\documentclass{article}
\PassOptionsToPackage{numbers}{natbib}
\usepackage{wrapfig}

\usepackage[preprint]{neurips_2026}

\usepackage[utf8]{inputenc} 
\usepackage[T1]{fontenc}    
\usepackage{hyperref}       
\usepackage{url}            
\usepackage{booktabs}       
\usepackage{amsfonts}       
\usepackage{nicefrac}       
\usepackage{microtype}      
\usepackage{xcolor}         

\usepackage{microtype}
\usepackage{graphicx}
\usepackage{subcaption}
\usepackage{enumitem}

\usepackage{amsmath}
\usepackage{amssymb}
\usepackage{mathtools}
\usepackage{amsthm}

\usepackage{algorithm}
\usepackage{algpseudocode}

\theoremstyle{plain}
\newtheorem{theorem}{Theorem}[section]
\newtheorem{proposition}[theorem]{Proposition}
\newtheorem{lemma}[theorem]{Lemma}
\newtheorem{corollary}[theorem]{Corollary}
\theoremstyle{definition}
\newtheorem{definition}[theorem]{Definition}

\theoremstyle{remark}

\newcommand{\bP}{\mathbb{P}}

\title{Differentially Private Model Merging}

\author{%
  Qichuan Yin\thanks{Corresponding author: \texttt{qichuan@uchicago.edu}} \\
  The University of Chicago \\
  \And
  Manzil Zaheer \\
  Google DeepMind \\
  \And
  Tian Li \\
  The University of Chicago \\
}

\begin{document}

\maketitle

\begin{abstract}
  In machine learning, privacy requirements at inference or deployment time often evolve due to changing policies, regulations, or user preferences. In this work, we aim to construct a magnitude of models to satisfy any target differential privacy (DP) requirement \textit{without additional training}, given a set of existing models trained on the same dataset with different privacy/utility tradeoffs. We propose two post-processing techniques, namely random selection and linear combination, to generate final private models satisfying any target privacy parameter. We provide privacy accounting of these approaches from the lens of R\'enyi DP and privacy loss distributions on general problems, as well as on private mean estimation, where we precisely characterize the privacy/utility tradeoffs and compare the two mechanisms. Empirically, we demonstrate the effectiveness of our approaches and validate our analyses on several models and both synthetic and real-world datasets. 
\end{abstract}

\vspace{-0.04in}
\section{Introduction}
\vspace{-0.04in}
Differential privacy (DP) ~\cite{dwork2006differential} offers a rigorous statistical framework to quantify privacy leakage of random algorithms and has been widely used in practice~\citep[e.g.,][]{abowd20222020}. During inference or deployment time of real-world applications, the privacy requirements can constantly change due to evolving regulatory and policy contexts, or dynamic user experience. It is therefore critical to update the deployed models quickly while satisfying the target privacy constraints.
In the presence of new target privacy constraints, re-training (or finetuning) a model with DP may need extra hyperparameter tuning (e.g.,
noise multiplier and gradient clipping threshold) and its implementation of per-gradient clipping can incur significant computation and systems costs~\cite{ponomareva2023dp}. 

On the other hand, it is not uncommon for service providers to have access to multiple DP models (e.g., trained with different hyperparameters) with varying privacy parameters. For instance, during model development stage, one may train a non-private model as a reference, and train a portfolio of private models (on the same dataset) to explore the privacy/utility tradeoff.
In this work, we therefore ask: \textit{can we merge the existing models with different privacy/utility tradeoffs to output a magnitude of models that satisfy certain target privacy requirement efficiently without touching raw data or performing any additional training?}


Though various variants of model merging have been extensively studied in prior literature~\citep[e.g.,][]{wortsman2022model}, it remains open how to merge private models given a target privacy, and how to perform tight accounting of privacy guarantees. 
For instance, existing privacy composition results (e.g., sequential or advanced composition~\cite{dwork2010boosting}) are not tailored to any model merging algorithms and may result in a suboptimal privacy/utility tradeoff. 
To this end, we propose two simple, data-independent algorithms that enable modular and  efficient combination of multiple private models to satisfy changing privacy requirements during inference time. We further provide 
tailored privacy accounting based on  R\'enyi differential privacy~\cite{mironov2017renyi} and privacy loss distributions~\cite{meiser2018tight,sommer2018privacy} for toy and general ML problems.

The first merging algorithm is  based on the intuitive idea of randomly outputting any model based on some probability distribution (named random selection, or \textbf{RS}). And the second algorithm is to  perform a linear combination of the models (named linear combination, or \textbf{LC}). That is to say, suppose the algorithm takes as input two models $\theta_1$ and $\theta_2$,\footnote{Our framework
handles $N$ models in the rest of the paper.} 
RS flips a coin independent of the data and outputs $\theta_1$ with
probability $\pi$, and $\theta_2$ otherwise.
LC outputs the deterministic result $\theta_\lambda=\lambda\theta_1+(1-\lambda)\theta_2$.
In both cases, we provide algorithms to choose mixing coefficients $\pi$ and $\lambda$, aiming to optimize utilities subject to the privacy constraint. Our contributions can be summarized as below.
\vspace{-0.5em}
\begin{itemize}[leftmargin=*]
  \item We study the problem of merging private models to satisfy dynamic privacy requirements during deployment. 
  We introduce two lightweight post-processing mechanisms, random selection (RS) and linear combination (LC), as well as algorithms for setting merging parameters.

  \item 
  We provide privacy accounting for both mechanisms based on R\'enyi differential privacy and privacy loss distributions on general problems. We also examine a case study on a mean estimation problem where linear combination dominates random selection. 

  \item 
  We validate our theoretical claims via empirical evaluation. Additionally, the experiments demonstrate that both RS and LC improve privacy/utility tradeoffs than those under naive privacy accounting in both synthetic and real-world datasets. 
\end{itemize}

\section{Related Work and Preliminaries}
\textbf{Model Merging.}
Model merging combines multiple trained models in parameter space and has emerged as a widely used technique for improving generalization, robustness, or multitask performance, with most existing work focusing on non-private settings. Existing approaches include simple parameter averaging such as model soups~\citep{wortsman2022model}, weighted averaging schemes such as Fisher-weighted merging~\citep{matena2022merging}, parameter-level merging rules such as TIES~\citep{yadav2023ties}, and checkpoint averaging methods such as stochastic weight averaging (SWA)~\citep{izmailov2018averaging}. Our LC procedure is closest in spirit to simple parameter averaging methods, but chooses the merging weights to satisfy a prescribed privacy target. We defer a more detailed discussion of model merging strategies to Appendix~\ref{app:model-merging}.


Closest to our setting, recent work on merging in private learning mainly studies checkpoint or model aggregation as post-processing tools to improve generalization under DP-SGD noise. This direction is particularly appealing because DP-SGD introduces additional noise and often yields less stable checkpoints.  \citet{indri2023can} adapt stochastic weight averaging \citep{izmailov2018averaging} to the DP setting, and find that averaging checkpoints can improve performance and reduce variance. \citet{shejwalkar2022recycling} study broader aggregations of intermediate checkpoints. 
However, most existing work primarily targets utility improvement rather than flexible privacy adaptation, and DP-SGD checkpoint merging generally cannot guarantee privacy beyond the least private checkpoint, as discussed in Appendix~\ref{subsec:rs-same-run}.
To the best of our knowledge, no prior work systematically studies merging private models under target DP constraints.


\textbf{Differential Privacy.} Differential privacy (DP)~\citep{dwork2006differential} (Definition~\ref{def:dp}) is a classic statistical framework of measuring privacy properties for randomized algorithms. In machine learning, the dominant approach to enforcing DP constraints is to use DP variants of stochastic gradient descent (DP-SGD~\citep{abadi2016deep}), which clip per-example gradients and add calibrated noise at each iteration.
\begin{definition}[Differential Privacy~\citep{dwork2006differential}] \label{def:dp}
A randomized mechanism $\mathcal{M}$ that operates on datasets consisting of individual datapoints is \emph{$(\varepsilon,\delta)$-DP} if for all (measurable) subsets $S$ of outputs and for all neighboring datasets $D,D^\prime$ that have the same size and differ in exactly one record, we have $
\Pr[\mathcal{M}(D)\in S] \le e^{\varepsilon}\Pr[\mathcal{M}(D^\prime)\in S] + \delta.$
When $\delta=0$, we simply say that $\mathcal{M}$ is \emph{$\varepsilon$-DP}.
\end{definition}
Popular privacy accounting techniques, such as those based on R\'enyi differential privacy (RDP) and privacy loss distributions (PLD)~\citep{mironov2017renyi,wang2019subsampled,dong2022gaussian,meiser2018tight,gopi2021numerical,koskela2021tight,doroshenko2022connect,sommer2018privacy}, yield tighter bounds on the cumulative privacy loss and have become standard in practical deployments of DP-SGD. 
In this work, we explore RDP and PLD accountants to audit privacy of our proposed approaches. For neighboring datasets $D$ and $D'$, RDP uses R\'enyi divergence of order $\alpha$ (denoted as $D_{\alpha}$)  to measure the difference between output distributions, forming a curve (\textit{RDP parameters} or \textit{profiles})
$\alpha\mapsto \varepsilon_\alpha$, where
\begin{equation}
\varepsilon_\alpha \;=\; \sup_{D\sim D'} \mathcal{D}_\alpha\!\big(\mathcal{M}(D)\,\|\,\mathcal{M}(D')\big). 
\end{equation}
The main benefit is that it composes additively across $T$ steps:
$\varepsilon_\alpha^{\mathrm{total}}=\sum_{t=1}^T \varepsilon_\alpha^{(t)}$. We can then convert $(\alpha, \varepsilon_{\alpha})$-RDP into standard $(\varepsilon,\delta)$-DP~\cite{mironov2017renyi}. 
PLD-based accounting instead works with the hockey-stick divergence
\begin{equation}
\mathcal{H}_{\varepsilon}(P,Q)\;\triangleq\;\sup_{S}\big\{P(S)-e^{\varepsilon}Q(S)\big\}. \label{eq:pld_divergence}
\end{equation}
Here $P$ and $Q$ denote the (pessimistic estimate of) output distributions of a
mechanism $\mathcal{M}$ on neighboring datasets $D$ and $D^\prime$, and $S$ ranges over all measurable events. We can track the \textit{PLD parameters} or \textit{profiles} via 
$\delta(\varepsilon)=\max\{\mathcal{H}_{\varepsilon}(P,Q),\mathcal{H}_{\varepsilon}(Q,P)\}$, and compose
privacy loss over steps via convolution~\cite{sommer2018privacy}. 
Although PLD often gives tighter numerical bounds, we also report RDP results because RDP yields simpler analytical expressions, is useful for theoretical comparisons (e.g., Section \ref{sec: mean estimation}), and is computationally lightweight and widely supported in DP pipelines.

\section{Problem Formulation} \label{sec:problem_formulation}

We consider $N$ differentially private algorithms $\mathcal{M}_1 ,\ldots ,\mathcal{M}_N$ applied on the same dataset $D$ and each outputting parameters $\theta_1(D), \ldots , \theta_N(D) \in \mathbb R^{p}$.
For $i\in\{1,\ldots,N\}$, the training mechanism $\mathcal{M}_i$ satisfies $(\varepsilon_i,\delta_i)$-DP (Definition~\ref{def:dp}).
Note that we allow one of the  models to be non-private, i.e., some $\varepsilon_i \to \infty$. 
Our goal is to construct a (possibly randomized) mechanism $h$ that takes as input these $N$ models along with the algorithms, and outputs
\[
\theta(D) \;=\; h\!\big(\{\theta_i(D)\}_{i=1}^N,\{\mathcal{M}_i\}_{i=1}^N; \varepsilon',\delta^\prime\big)
\]
that satisfies target $(\varepsilon',\delta^\prime)$-DP.
Obviously, for suitably large $(\varepsilon', \delta^\prime)$, there always exists a trivial solution set $h\!\big(\{\theta_i(D)\}_{i=1}^N,\{\mathcal{M}_i\}_{i=1}^N; \varepsilon',\delta^\prime\big)\in 
\{h: h \text{ depends only on the most private } \theta_i(D) \,\}$.
Therefore, rather than identifying a single construction, we aim to characterize a large class of feasible merged outputs with as tight privacy accounting as possible, which translates to better privacy/utility tradeoffs. 

\textbf{Necessity of Knowing $\{\mathcal{M}_i\}_{i=1}^N$.}
Ideally, we would like the merging algorithm $h$ to depend only on the $N$ model parameters $\{\theta_i(D)\}_{i=1}^N$ without having access to the hyperparameters of the training algorithms $\{\mathcal{M}_i\}_{i=1}^N$ that produce them. However, such mechanism-agnostic post-processing merging algorithms may be fundamentally limited. In particular, next we show that when $N=2$, under mild regularity conditions, any $h$ that outputs an $(\varepsilon',\delta')$-DP model  must lie very close to the class 
$
\{h: h \text{ depends only on the more private } \theta_i(D) \,\}.
$ 

\begin{proposition}[No-free-lunch without knowing $\{\mathcal{M}_i\}_{i=1}^N$]
\label{prop:impossibility}
Assume the mapping $h=h(\theta_1(D),\theta_2(D); \varepsilon',\delta^\prime)$ (either deterministic or random) is measurable and independent of ($\mathcal{M}_1,\mathcal{M}_2$). 
Assume $\varepsilon_1<\varepsilon^\prime<\varepsilon_2$ and $\delta_1=\delta_2=\delta$. 
If 
\begin{equation}\label{cond:impossible}
    \sup_{x,a,b\in \mathbb R^p
} \mathcal{D}_{TV}(h(x,a), h(x,b))>\frac{(\delta^\prime+e^{\varepsilon^\prime}-1)(e^{\varepsilon_2}+1)}{e^{\varepsilon_2+\varepsilon^\prime}+\delta e^{\varepsilon^\prime}+\delta-1},
\end{equation}
then the output $h\!\big(\theta_1(D),\theta_2(D)\big)$ cannot be $(\varepsilon',\delta')$-DP for all admissible choices of $(\mathcal{M}_1,\mathcal{M}_2)$.
\end{proposition}

The complete proof is provided in Appendix \ref{app:impossibility}.
Condition~\ref{cond:impossible} describes mergers $h(\cdot,\cdot)$ that react too much to changes in the less private model $\theta_2$. 
If varying $\theta_2$ while fixing $\theta_1$  can induce a large total-variation change in the distribution of $h(\theta_1,\theta_2)$, then no universal privacy guarantee is possible.
Proposition~\ref{prop:impossibility} formalizes that, without structural assumptions on $\mathcal{M}_1,\mathcal{M}_2$, one cannot design a merging algorithm that is substantially different from the trivial merger while satisfying $(\varepsilon',\delta')$-DP. 
In particular, when $\delta'=\delta$, any deterministic merger that aims to be valid uniformly over $(\mathcal{M}_1,\mathcal{M}_2)$ must lie in the trivial class.

Therefore, to obtain meaningful improvements beyond the trivial merger, we focus on practically relevant mechanisms rather than arbitrary ones. In particular, we study the setting where the inputs are associated with some known RDP or PLD parameters (which can be inferred from hyperparameters of the learning algorithms) (Section~\ref{sec: RS}), or where the inputs are directly produced by running DP-SGD (Section~\ref{sec:linear-combination}). 
In the next sections, we first formally introduce two model merging algorithms (two instances of $h$), random selection and linear combination (Section~\ref{sec:method}). 
We then discuss privacy accounting of RS (Section~\ref{sec: RS}) and LC (Section~\ref{sec:linear-combination}) mechanisms, and consider how to choose an appropriate merger $h$ that exploits the additional assumptions of $\mathcal{M}$. 

\section{Model Merging via Random Selection or Linear Combination} \label{sec:method}
In this section, we first formally propose two data-independent merging mechanisms, and illustrate the privacy/utility tradeoff introduced by these mechanisms on a Gaussian mean estimation problem (Section~\ref{sec: mean estimation}).

\textbf{Random Selection (RS).} We propose a mechanism  independent of $D$ that samples a model index $I$ from a categorical distribution and outputs $\theta_I(D)$ as the merged model. That is, given target privacy $(\varepsilon',\delta')$, we output  
\begin{equation}
\theta(D)
= \theta_I(D), \quad
\ I \sim \mathrm{Cat}(\pi_1,\ldots,\pi_N),  \label{eq:RS}
\end{equation}
where 
$\pi=(\pi_1,\ldots,\pi_N)\in\Delta^{N-1}$ (a probability simplex) and $\pi$ is dependent on $(\varepsilon', \delta')$ (Section~\ref{sec: RS}). 
Note that RS only randomly samples one model following some distribution. Another natural choice is to leverage all the models by linearly averaging them.

\textbf{Linear Combination (LC).} The LC method for $N$ models on dataset $D$ is defined as
\vspace{-0.5em}
\begin{equation}
\theta(D) \;=\; \sum_{i=1}^N\lambda_i\,\theta_i(D),
~ \lambda=[\lambda_1, \cdots, \lambda_N] \in \Delta^{N-1}.
\label{eq:def-linear-comb}
\end{equation}
In both cases, as  $h$ does not have access to $D$, 
the privacy of composing $N$ models is fully determined by how the privacy losses of $\{\mathcal{M}_i\}_{i=1}^N$ are aggregated under different instantiations of $h$, and the main task becomes privacy accounting for the RS or LC procedure. 

At a high level, RS is entirely characterized by privacy parameters of  RDP or PLD, as the output model is randomly sampled from the input $N$ models. Therefore, its privacy accounting depends on the individual selected model's privacy parameters, making RS broadly applicable even when we do not understand the detailed procedure of the training algorithm (as long as we have access to privacy parameters).
In contrast, LC requires to imposing stronger assumptions on the structures (e.g., hyperparameters and updating rules) of the training algorithms that produce private input models (formalized in Section~\ref{sec:linear-combination}). 
Next, we study a  toy problem on a Gaussian mean estimation to illustrate this intuition, before introducing general accounting.

\subsection{Case Study: Mean Estimation} \label{sec: mean estimation}
This section studies  Gaussian mean estimation, one of the simplest statistical tasks where we can precisely characterize the privacy and utilities of RS and LC, under RDP accounting.  
In this setting, we show that LC can match the privacy of RS while achieving no larger MSE. 
Let $\hat\mu$ be the empirical mean with $\ell_2$ sensitivity $\Delta$, and consider two independent Gaussian mechanisms
$
\theta_i(D)\sim \mathcal N(\hat\mu,\sigma_i^2),\quad i=1,2.
$
For RS with parameter $\pi$, the output follows
$
\pi\mathcal N(\hat\mu,\sigma_1^2)+(1-\pi)\mathcal N(\hat\mu,\sigma_2^2),
$
with variance $\sigma_{\mathrm{RS}}^2(\pi)=\pi\sigma_1^2+(1-\pi)\sigma_2^2$. 
For LC with parameter $\lambda$, the output follows
$
\mathcal N\!\left(\hat\mu,\lambda^2\sigma_1^2+(1-\lambda)^2\sigma_2^2\right),
$
with variance $\sigma_{\mathrm{LC}}^2(\lambda)=\lambda^2\sigma_1^2+(1-\lambda)^2\sigma_2^2$. 
Since both estimators are unbiased relative to $\hat\mu$, comparing their MSE reduces to comparing these variances.

\begin{theorem}
\label{thm:LC-dominates-alpha}
Fix any $\alpha>1$ and $\pi\in[0,1]$. There exists a constant $C_{\pi,\alpha}>0$ such that, whenever
$\Delta<C_{\pi,\alpha}$, the following holds: if
$\sigma_{\mathrm{LC}}^2(\lambda)=\sigma_{\mathrm{RS}}^2(\pi)$, then LC
with parameter $\lambda$ achieves privacy no worse than RS with parameter
$\pi$ under order-$\alpha$ RDP.
Consequently, for any RS parameter $\pi$, there exists an LC parameter
$\lambda$ that achieves no larger MSE under the same target order-$\alpha$
RDP budget.
\end{theorem}

Thus, in this Gaussian setting, LC improves over RS as it averages independent noise rather than discarding one candidate. This clean comparison relies on the additive Gaussian structure and does not directly extend to general learning problems. We next develop accounting rules for RS under general RDP/PLD profiles, and for LC under the additional structure of independent DP-SGD runs.

\section{Privacy of Random Selection}\label{sec: RS}
Recall that we have trained $N$ private models on the same dataset $D$, obtaining output parameters
$\theta_1(D), \ldots, \theta_N(D)\in\mathbb{R}^p$.
As discussed in Section~\ref{sec:problem_formulation}, there is no general post-processing algorithm
that improves utilities beyond the degenerated solutions if we do not have access to any parameters of the training algorithms $\{\mathcal{M}_i\}_{i=1}^N$.
Hence, for random selection, we assume that we are given the hyperparameters of $\mathcal{M}$, from which we derive the RDP and PLD profiles of each input model. For each $\mathcal{M}_i$, we  denote the RDP parameters as $\{\varepsilon_{\alpha,i}\}_{\alpha>1}$  and the PLD privacy parameters as $\varepsilon\mapsto\delta_i(\varepsilon)$.
To this end, we develop two complementary accounting techniques: an RDP-based bound that is convenient for analysis and optimization, and a PLD-based method that is typically tighter in practice.

\textbf{RDP Accounting.}
We begin by establishing an explicit upper bound on the RDP parameters of RS.
\begin{theorem}
\label{prop:rdp-mixture}
For any $\alpha>1$, random selection (Eq.~\eqref{eq:RS}) over $N$ private models (each being $(\alpha, \varepsilon_{\alpha,i})$-RDP) satisfy $(\alpha, \varepsilon_{\alpha}^{\text{RS}}(\pi)$)-RDP where 
$\varepsilon_{\alpha}^{\text{RS}}(\pi)
\;\le\;
\frac{1}{\alpha-1}\,
\log\!\Big( \sum_{i=1}^{N}
\pi_i\,e^{(\alpha-1)\varepsilon_{\alpha,i}}
\Big).$
\end{theorem}

Theorem~\ref{prop:rdp-mixture} (proved in Appendix~\ref{app:proof:rs_rdp}) provides an explicit upper bound on the RDP parameter of random selection in terms of the individual model's parameters $\{\varepsilon_{\alpha,i}\}_{i=1}^N$ and the mixing
weights $\pi\in\Delta^{N-1}$. This bound immediately enables us to certify
$(\varepsilon(\pi;\delta^{\prime}), \delta^{\prime})$-DP for candidate mixtures via the standard RDP-to-DP conversion.
For a target privacy budget $\varepsilon'$, we accept a mixing weight
$\pi$ if $\varepsilon(\pi;\delta^\prime)\le \varepsilon'$. This provides us the criterion for whether a candidate 
$\pi$ satisfies the target privacy guarantee.
We summarize the method in Algorithm \ref{algo:RS-combined} (left).




\setlength{\textfloatsep}{5pt}
\setlength{\abovecaptionskip}{2pt}
\setlength{\belowcaptionskip}{2pt}
\begin{algorithm}[t]
\caption{Setting Selection Probabilities for RS (See Algorithm~\ref{algo:LC-combined} in the appendix for LC)}
\label{algo:RS-combined}
\begin{algorithmic}[1]

\State {\bfseries Input:} target privacy $(\varepsilon', \delta')$, orders grid $\mathcal{A}$ for RDP, 
       hyperparameter sets $\{\mathcal{S}_i\}_{i=1}^N$ of $\{\mathcal{M}_i\}_{i=1}^N$.
\State {\bfseries Output:} set of selection probabilities $\pi\in\Delta^{N-1}$ satisfying the target privacy.
\vspace{-0.07in}
\Statex \rule{\linewidth}{0.4pt}  
\Statex\noindent  
\begin{minipage}[t]{0.51\linewidth}
  \textit{Case 1: RDP accounting}
  \begin{algorithmic}[1]   
    \State  Obtain
           $\{\varepsilon_{\alpha,i}\}_{\alpha\in\mathcal{A}}$ by $\mathcal{S}_i$, for each $i \in [N]$
    \Function{DP\_Eps}{$\pi,\delta$}
      \For{$\alpha \in \mathcal{A}$}
        \State $a_i$ $\gets$ $(\alpha-1)\,\varepsilon_{\alpha,i}$, $i\in[N]$
        \State $\varepsilon_{\alpha}^{\text{RS}}(\pi) $ $\gets$
               $\tfrac{\log\!\bigl(\sum_{i=1}^N \pi_i e^{a_i}\bigr)}{\alpha-1}$
        \State $\varepsilon_{\alpha}(\pi;\delta)$ $\gets$
               $\varepsilon_{\alpha}^{\text{RS}}(\pi)
               + \tfrac{\log(1/\delta)}{\alpha-1}$
      \EndFor
      \State \textbf{return}
             $\min_{\alpha\in\mathcal{A}}\varepsilon_{\alpha}(\pi;\delta)$
    \EndFunction
    \State {\bfseries return} $\{\pi \mid \text{DP\_Eps($\pi,\delta'$)}\leq \varepsilon^\prime\}$
  \end{algorithmic}
\end{minipage}%
\hfill%
\begin{minipage}[t]{0.49\linewidth}
  \textit{Case 2: PLD accounting}
  \begin{algorithmic}[1]
    \State Obtain
           $\varepsilon\mapsto\{\mathcal{H}_\varepsilon(P_i,Q_i),\,
           \mathcal{H}_\varepsilon(Q_i,P_i)\}$ by $\mathcal{S}_i$, for each $i \in [N]$. 
    \Function{DP\_Delta}{$\pi,\varepsilon$}
      \State $\delta_{+} \gets
             \sum_{i=1}^N \pi_i\,\mathcal{H}_\varepsilon(P_i,Q_i)$
      \State $\delta_{-} \gets
             \sum_{i=1}^N \pi_i\,\mathcal{H}_\varepsilon(Q_i,P_i)$
      \State $\delta^{\mathrm{RS}}(\pi;\varepsilon) \gets
             \max\{\delta_{+},\delta_{-}\}$
      \State \textbf{return} $\delta^{\mathrm{RS}}(\pi;\varepsilon)$
    \EndFunction
    \State {\bfseries return} $\{\pi \mid \text{DP\_Delta($\pi,\varepsilon^\prime$)}\leq \delta^\prime\}$
  \end{algorithmic}
\end{minipage}
\end{algorithmic}
\end{algorithm}

\textbf{PLD Accounting.}
In the PLD accounting, we note that random selection induces the mixture distributions over the $N$ models. That is, considering any neighboring datasets $D$ and $D^\prime$, the output probabilities of RS are
\vspace{-0.06in}
\begin{align*}
    P_{RS}^{D} =\sum_{i=1}^N\pi_i P_i^{D}, \quad P_i^{D}=\text{Law}(\mathcal{M}_i(D)), \quad
Q_{RS}^{D^\prime} =\sum_{i=1}^N\pi_iQ_i^{D^\prime}, \quad Q_i^{D^\prime}=\text{Law}(\mathcal{M}_i(D^\prime))
\end{align*}

where $\text{Law}(X)$ refers to the distribution induced by random variable $X$. Further, we define $P_i$ and $Q_i$ to be pessimistic estimates of $P_i^{D}$ and $Q_i^{D^\prime}$ respectively, which means $\mathcal{H}_{\varepsilon}(P_i^{D},Q_i^{D^\prime}) \leq \mathcal{H}_{\varepsilon}(P_i,Q_i)$ and $\mathcal{H}_{\varepsilon}(Q_i^{D^\prime},P_i^{D}) \leq \mathcal{H}_{\varepsilon}(Q_i,P_i)$. 
Note that the hockey-stick divergence used by PLD accounting $\mathcal{H}_{\varepsilon}(\cdot,\cdot)$ (Eq.~\eqref{eq:pld_divergence}) is jointly convex over $(P,Q)$. Then we have the following result about the PLD parameter  $\delta^{\text{RS}}(\varepsilon)$ of the random selection output.

\begin{theorem}
\label{prop:pld-convex-mixture}
For any $\varepsilon$$\ge$ 0 and $\pi$$=$$(\pi_1,\ldots,\pi_N)\in\Delta^{N-1}$, it holds that
\vspace{-0.06in}
\[
\sup_{D\sim D ^\prime}\mathcal{H}_{\varepsilon}(P_{RS}^{D},Q_{RS}^{D^\prime})
\leq  \sum_{i=1}^N\pi_i\,\mathcal{H}_{\varepsilon}(P_i,Q_i), \quad \sup_{D\sim D ^\prime}\mathcal{H}_{\varepsilon}(Q_{RS}^{D^\prime},P_{RS}^{D})
\;\le\;
\sum_{i=1}^N\pi_i\,\mathcal{H}_{\varepsilon}(Q_i,P_i).
\]
\vspace{-0.07in}
Consequently, we have $\delta^{\text{RS}}(\varepsilon)$ is upper bounded by 
$
 \max \!\Big\{
\sum_{i=1}^N\pi_i \mathcal{H}_{\varepsilon}(P_i,Q_i),\sum_{i=1}^N\pi_i  \mathcal{H}_{\varepsilon}(Q_i,P_i)
\Big\}$.
\end{theorem}

Theorem~\ref{prop:pld-convex-mixture} (proved in Appendix~\ref{app:proof:rs_pld}) also yields a simple feasibility test for RS based on the privacy loss distribution.
Given a target privacy parameter $\varepsilon^\prime$, we evaluate $\delta^{\text{RS}}(\varepsilon^\prime)$ for a candidate probability $\pi$ and $\pi$ is feasible if and only if $\delta^{\text{RS}}(\varepsilon^\prime)\leq \delta^\prime$.
We summarize this method in Algorithm \ref{algo:RS-combined} (right).

\textbf{Privacy/Utility Tradeoff under Random Selection. } 
For a target $(\varepsilon',\delta')$, the accounting methods above provide a feasible set of $\pi\in\Delta^{N-1}$. In practice, one would naturally like to choose $\pi$ that yields a strong utility among all feasible $\pi$'s. This calls for selection rules over $\pi$. Under random selection, the expected utility is
linear in $\pi$, making the tradeoff easier to reason about. 
In typical DP scenarios, more private models tend to have lower utilities, so a natural strategy is to choose $\pi$ that nearly saturates the privacy budget,
e.g., $\varepsilon(\pi;\delta')\approx \varepsilon'$ in RDP, thereby biasing selection towards higher-utility models while
maintaining the privacy guarantee.
Moreover, many selection rules are possible in practice. For example, one may evaluate candidate mixtures on a public dataset and choose the best-performing option. 





\vspace{-0.04in}
\section{Privacy of Linear Combination}
\label{sec:linear-combination}
\vspace{-0.04in}


In Section~\ref{sec: mean estimation}, we leverage the additivity of independent Gaussian noise and show that, for mean estimation with a Gaussian mechanism, LC strictly dominates RS. However, such favorable structure does not hold for general problems. For random selection, the only assumption we make is access to the RDP/PLD parameters of each input model. In
contrast, we show in this section that this is not sufficient for linear combination. That is, if we only know the privacy parameters and impose no further assumptions, then the privacy under
LC cannot improve upon that of simply releasing all the models $\{\theta_1,\ldots,\theta_N\}$ jointly. This observation (formalized in Section~\ref{subsec:linmix-limitation} below) motivates the need
for additional structural assumptions to obtain meaningful privacy accounting for LC (Section~\ref{sec:method:lc}).

\subsection{Limitation Result without Knowing Original Training Algorithms} 
\label{subsec:linmix-limitation}
Suppose we \textit{only} know the final RDP or PLD privacy parameters of each private learning mechanism $\{\mathcal{M}_i\}_{i=1}^N$ without any further algorithmic structure information (i.e., the exact updating rules) about $\{\mathcal{M}_i\}_{i=1}^N$. 
We show below, when the mechanisms are instantiated with fresh independent randomness conditioned on the dataset $D$, the LC mechanism admits a universal worst-case upper bound, and this bound can be achieved.

\begin{theorem}[Worst-case bounds and tightness]
\label{thm:linmix-profile-bound}
Consider the linear combination mechanism defined by
$
\theta(D)=\sum_{i=1}^N \lambda_i\,\theta_i(D).
$
Assume $\lambda_i>0$ and the inputs $\{\theta_1(D), \cdots, \theta_N(D)\}$ are generated with fresh independent randomness conditioned on the dataset $D$.
We have the following results.

\textbf{(1) RDP upper bound.}
The RDP parameter of $\theta$ is no worse than that of the joint release (i.e., releasing $\{\theta_1,\cdots, \theta_N\}$). Thus, for every $\alpha>1$, the LC mechanism satisfies $(\alpha, \varepsilon^{\text{LC}}_\alpha(\lambda))$-RDP where
\begin{equation}
\varepsilon^{\text{LC}}_\alpha(\lambda)
\;\le\;
\sum_{i=1}^N \varepsilon_{\alpha,i}.
\label{eq:linmix-rdp-sum}
\end{equation}
\textbf{(2) PLD upper bound.}
Let $\omega_i$ be the PLD of $\theta_i$, and let $\omega_{\mathrm{joint}}=\omega_1\star\cdots\star\omega_N$
be the PLD of the joint release. Then the PLD parameter of $\theta$ is no worse than that of the joint release.
That is, for any $\varepsilon\ge 0$,
\begin{equation}
\delta^{\text{LC}}(\varepsilon)
\;\le\;
\delta_{\mathrm{joint}}(\varepsilon),
\label{eq:linmix-pld-upper}
\end{equation}
where $\delta_{\mathrm{joint}}(\varepsilon)$ is computed from $\omega_{\mathrm{joint}}$.

\textbf{(3) Tightness.}
The bounds in Eq. \eqref{eq:linmix-rdp-sum}-\eqref{eq:linmix-pld-upper} are worst-case tight: there exist mechanisms
$\{\mathcal{M}_i\}_{i=1}^N$ under which the equality holds. 
\end{theorem}
Theorem~\ref{thm:linmix-profile-bound} (proved in Appendix~\ref{app:proof:lc_worst_case}) establishes that the privacy parameter of the deterministic linear combination is fully controlled by that of the joint model, without additional assumptions.
Moreover, the bound is worst-case tight; so LC satisfies no better privacy guarantee gain beyond the joint release $\{\theta_1,\cdots, \theta_N\}$. 
Therefore, without additional information beyond the input models' privacy parameters, one cannot derive a universally tighter privacy guarantee for LC. 
In particular, the resulting joint-release bounds are
substantially larger than the RS bounds in Theorems~\ref{prop:rdp-mixture} and~\ref{prop:pld-convex-mixture}, which are always
upper-bounded by the maximum RDP/PLD parameter among all input models. This highlights that, without further
structure, LC is not guaranteed to outperform RS.

\vspace{-0.04in}
\subsection{Linear Combination under DP-SGD} \label{sec:method:lc}
\vspace{-0.04in}

Here, we assume that the input models are trained independently via DP-SGD~\cite{abadi2016deep}, one of the most popular DP  algorithms in ML. For clarity and ease of analysis, we consider the full-batch version.  

At step \(t=1,\dots,T\), each mechanism \(i\in[N]\) computes a clipped gradient
$
\tilde g_i^{(t)}=\mathrm{clip}\big(g_i^{(t)},C_i^{(t)}\big)
$
such that $\|\tilde g_i^{(t)}\|_2\le C_i^{(t)}$, and add Gaussian noise
$
e_i^{(t)}\sim \mathcal N\!\Big(0,\big(\sigma_i^{(t)}C_i^{(t)}\big)^2 I\Big),
$
independently across \(i\). The local update is
\[
\theta_i^{(t)}=\theta_i^{(t-1)}-\eta_i^{(t)}\big(\tilde g_i^{(t)}+e_i^{(t)}\big).
\]
Define the model output via LC after step $t$ as $
\theta^{(t)}=\sum_{i=1}^N\lambda_i\,\theta_i^{(t)}.
$
The combined update can therefore be written as
{
\setlength\abovedisplayskip{0pt}
\setlength\belowdisplayskip{0pt}
\[
\theta^{(t)}-\theta^{(t-1)}
\;=\;
-\sum_{i=1}^N\lambda_i\,\eta_i^{(t)}\big(\tilde g_i^{(t)}+e_i^{(t)}\big)
\;:=\;
-\big(\overline g_\lambda^{(t)}+\overline e_\lambda^{(t)}\big),
\]
}
where $\overline g_\lambda^{(t)}$$:=$$\sum_{i=1}^N\lambda_i\eta_i^{(t)}\tilde g_i^{(t)}$ and
$\overline e_\lambda^{(t)}$$:=$$\sum_{i=1}^N\lambda_i\eta_i^{(t)}e_i^{(t)}$.
Note that the aggregated noise $\overline e_\lambda^{(t)}$ follows the distribution $N\!\big(0,(s_\lambda^{(t)})^2I\big)$, where
$(s_\lambda^{(t)})^2$$:=$$\sum_{i=1}^N\big(\lambda_i\,\eta_i^{(t)}\,\sigma_i^{(t)}\,C_i^{(t)}\big)^2.$
Define the mixed sensitivity 
$\Delta_\lambda^{(t)}=\sum_{i=1}^N \lambda_i \eta_i^{(t)} \Delta_i^{(t)}$, where $\Delta_i^{(t)}$ is the sensitivity of mechanism \(i\) for step $t$.

A subtle point is that standard RDP and PLD composition (shown in Appendix \ref{app:composition theory}) is applied by conditioning on the previously accounted output. 
In our LC analysis, this means conditioning on the previous merged parameter \(\theta^{(t-1)}\), rather than on the individual parameter  \(\{\theta_i^{(t-1)}\}_{i=1}^N\). 
Consequently, given dataset $D$, the individual gradients \(\widetilde g_i^{(t)}\) are not fixed after conditioning on \(\theta^{(t-1)}\); the conditional distribution of the combined update \(\theta^{(t)}-\theta^{(t-1)}\) is generally a mixture over all possible values of \(\{\theta_i^{(t-1)}\}_{i=1}^N\), rather than a single Gaussian. 

Let $\mathcal F_{t-1}$ denote the sigma field generated by $\theta^{(t-1)}$.
The following results show that, although the conditional distribution of \(\{\theta_i^{(t-1)}\}_{i=1}^N\) is unknown and the resulting combined update is generally non-Gaussian, valid RDP and PLD accounting bounds for LC can still be obtained.

\begin{theorem}
\label{thm:rdp-linmix}
For any order $\alpha>1$, conditional on $\mathcal{F}_{t-1}$, the RDP bound at step-$t$ for the LC mechanism satisfies:
\begin{equation}\label{ineq:perstep-LC}
\varepsilon_\alpha^{(t)}(\lambda)
\leq 
\alpha(\Delta_\lambda^{(t)})^2 / \left(2(s_\lambda^{(t)})^2\right).
\end{equation}
\vspace{-0.5em}
\end{theorem}
Full proof of Theorem~\ref{thm:rdp-linmix} is in Appendix~\ref{app:proof:rdp-linmix}. In practice, we apply Theorem~\ref{thm:rdp-linmix} by summing the per-step privacy costs over all $T$ iterations using the standard RDP composition rule in Appendix \ref{app:composition theory}, and
then converting the resulting RDP parameters into DP guarantees. 
We summarize the accounting procedures in Algorithm \ref{algo:LC-combined} (left) in the appendix. 
For PLD accounting, we have the following results.

\begin{theorem}
\label{thm:pld-linmix}
     Let $P^{(t)}$ and $Q^{(t)}$ denote the distribution of $\theta^{(t)}-\theta^{(t-1)}$ under $D$ and $D'$ conditional on $\mathcal{F}_{t-1}$. Consider the practical surrogate privacy loss random variable
{\setlength\abovedisplayskip{0pt}
\setlength\belowdisplayskip{0pt}
\begin{equation}
\label{eq:randomv-LC-surrogate}
\tilde L^{(t)}
=
\log
\frac{
\phi\!\left(
Z^{(t)}; \Delta_\lambda^{(t)}, (s_\lambda^{(t)})^2
\right)
}{
\phi\!\left(
Z^{(t)}; 0, (s_\lambda^{(t)})^2
\right)
}, \qquad Z^{(t)}\sim \mathcal N(0,(s_\lambda^{(t)})^2),
\end{equation}}
where $\phi(\cdot;\mu,\sigma^2)$ denotes Gaussian density with mean $\mu$ and variance $\sigma^2$.
Then, for every $\varepsilon \ge 0$, 
\begin{align*}
\mathcal H_\varepsilon(P^{(t)},Q^{(t)})
\le
\mathbb E[(e^{\tilde L^{(t)}}-e^\varepsilon)_+],
\quad
\mathcal H_\varepsilon(Q^{(t)},P^{(t)})
\le
\mathbb E[(1-e^{\varepsilon+\tilde L^{(t)}})_+].
\end{align*}
\end{theorem}

\begin{corollary}
\label{cor:pld-linmix-compose}
Under the notation of Theorem~\ref{thm:pld-linmix}, assume $L^{(t)}$ is the real privacy loss random variable. 
Define
$
L := \sum_{t=1}^T L^{(t)},
\tilde L := \sum_{t=1}^T \tilde L^{(t)}.
$
Then, for every $\varepsilon \ge 0$, it holds that 
$$
\mathbb E\!\left[\bigl(e^{L}-e^\varepsilon\bigr)_+\right]
\;\le\;
\mathbb E\!\left[\bigl(e^{\tilde L}-e^\varepsilon\bigr)_+\right], \quad
\mathbb E\!\left[\bigl(1-e^{\varepsilon+L}\bigr)_+\right]
\;\le\;
\mathbb E\!\left[\bigl(1-e^{\varepsilon+\tilde L}\bigr)_+\right]
$$
\end{corollary}



As discussed above, the conditional distribution of 
$\theta^{(t)}-\theta^{(t-1)}$ is generally non-Gaussian, so the actual privacy-loss random variable $L^{(t)}$ does not admit a simple form. 
The surrogate variable $\widetilde L^{(t)}$ in Theorem~\ref{thm:pld-linmix} gives a tractable Gaussian upper bound for each step. 
Together with Corollary~\ref{cor:pld-linmix-compose} and the standard PLD composition rule in Appendix~\ref{app:composition theory}, this yields a practical LC accountant by convolving the per-step surrogate PLDs from Eq.~\eqref{eq:randomv-LC-surrogate}. 
Algorithm~\ref{algo:LC-combined} (right) summarizes the procedure; complete proofs are in Appendix~\ref{app:proof:pld-linmix}.

While LC does not enjoy the same clean utility/privacy tradeoff structure as RS in general, especially
for nonconvex models, many practical selection rules remain applicable. In particular, LC can still be
tuned effectively by evaluating performance on a public benchmark or a non-sensitive
validation set.

\textbf{Different Training Steps Across Input Models.}
In some applications, the  input candidate models may be trained with differential privacy for different numbers of iterations 
$\{T_i\}_{i=1}^N$. To make the accounting go through, we can align all runs to a common iteration
$T \triangleq \max_i T_i$. For any model $i$ with $T_i<T$, we could use an equivalent $T$-step representation by appending $T-T_i$ virtual steps.

\section{Comparison with Related DP Merging and Composition Methods}
\textbf{Recycling Checkpoints.}
Recent works \citep{shejwalkar2022recycling,indri2023can} improve private learning by aggregating  checkpoints from a single DP-SGD run. This algorithm differs fundamentally from ours because the candidate models share randomness and early optimization steps, thus affecting accounting. Our RS accounting only requires the individual RDP/PLD profiles of each candidate model, and therefore remains applicable to recycled checkpoints; see Appendix~\ref{subsec:rs-same-run} for empirical results. In contrast, our LC accounting relies on additional structure, in particular independent DP-SGD noise across input models. These assumptions fail for correlated checkpoints, so our LC guarantees do not apply in this setting. More specifically, LC over recycled checkpoints cannot guarantee privacy better than the least private checkpoint; see Appendix~\ref{subsec:rs-same-run} for details.

\textbf{Comparing with Other DP Composition.}
Classical DP composition bounds \citep[e.g.,][]{dwork2010boosting,kairouz2015composition} analyze the joint release of multiple private outputs using only their $(\varepsilon_i,\delta_i)$ guarantees. In contrast, our accounting uses richer privacy profiles, such as RDP parameters across orders or PLD curves, which typically yield tighter guarantees. Moreover, our bounds exploit the structure of the merging rule itself: RS uses the mixture structure of random selection through log-sum-exp and convexity, while LC can obtain tighter bounds when the input models are produced by independent DP-SGD runs. Thus, our accounting can be sharper than generic joint-release accounting, while also characterizing how the merging weights affect privacy. A formal comparison is provided in Proposition~\ref{prop:rdp-vs-drv10-gaussian}, where we show that RDP joint accounting improves over advanced composition.
\vspace{-.5em}

\vspace{-0.04in}
\section{Empirical Evaluation} \label{sec:exps}
\vspace{-0.04in}

In this section, we first present results on a toy mean estimation problem for RS and LC algorithms under both RDP and PLD accounting, corresponding to analysis in Section~\ref{sec: mean estimation}. Then on real-world datasets with both convex and non-convex models, we present the privacy/utility tradeoffs of a set of models output by RS and LC, and demonstrate tighter privacy bounds than baseline composition methods. Additional experiment details are discussed in Appendix \ref{app:experiment details}.

\vspace{-0.04in}
\subsection{Synthetic Data: Mean Estimation} \label{sec:exp-mean-est}
\vspace{-0.04in}

\begin{wrapfigure}[]{r}{0.45\textwidth} 
  \centering
  \vspace{-1em}
  \includegraphics[width=0.45\textwidth]{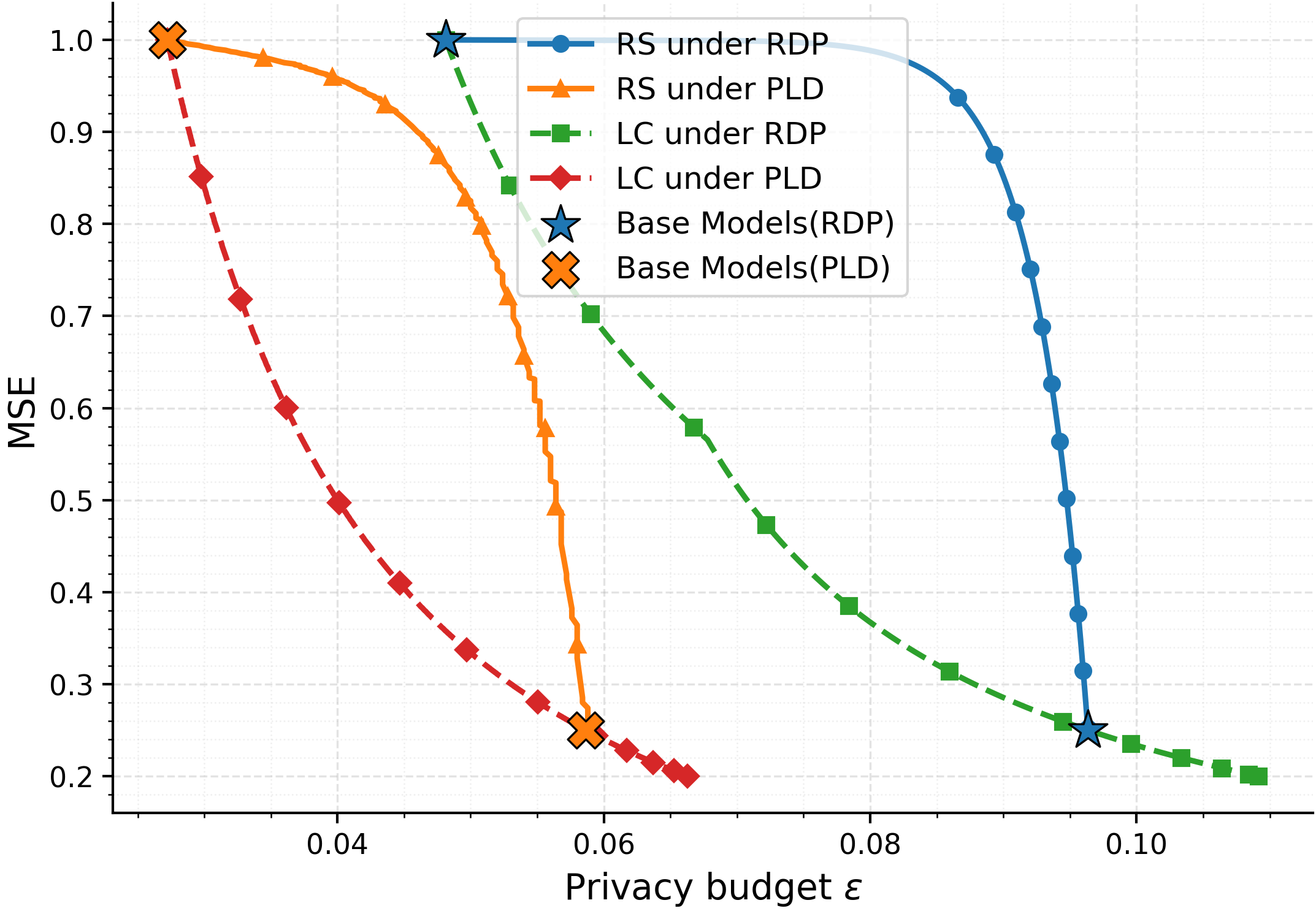}
  \caption{\small Privacy/utility tradeoffs of mean estimation ($\delta=10^{-5}$). Input models are also marked in the figure.}
  \label{fig:Mean-est}
  \vspace{-1em}
\end{wrapfigure}

We generate 
$X_1,\dots,X_n \overset{\mathrm{i.i.d.}}{\sim}\mathcal{N}(0,1)$ with $n=100$, and then clip the samples to enforce a
bounded domain, i.e., $X_j\leftarrow \mathrm{clip}(X_j,-1,1)$. The goal is to estimate the population mean $\mu$ for dataset
$D=\{X_1,\dots,X_n\}$. We generate two private estimates as input.
Figure~\ref{fig:Mean-est} reports the Pareto frontier between losses (mean squared error) and privacy budgets of the RS and LC algorithms under RDP and PLD accounting. PLD gives tighter privacy upper bounds than RDP.
Also, both methods achieve flexible privacy by tracing out a continuous MSE/privacy tradeoff as the target privacy level changes.
Moreover, LC consistently outperforms RS, validating our theoretical arguments in Section \ref{sec: mean estimation}.

\vspace{-0.04in}
\subsection{Results on Real Datasets}
\vspace{-0.04in}

We next evaluate our method on two standard benchmarks: MNIST~\citep{lecun2002gradient} and CIFAR-10 \citep{krizhevsky2009learning}, and train both convex and non-convex models.
We use the standard split of MNIST for training and testing, and train a logistic regression model with DP-SGD. 
For CIFAR-10, we use the standard training/test  split with a ResNet18 model. For both datasets, we train three input private models to cover two common sources of variation in practice. 
One model is obtained by changing a hyperparameter that does not affect the privacy accounting, and therefore has the same privacy guarantee but may exhibit different utility. 
Another model is obtained by changing a privacy-relevant hyperparameter, resulting in a different privacy/utility tradeoff. See Appendix~\ref{app:experiment details} for hyperparameters and implementation details.

\begin{figure}[h!]
\addtocounter{figure}{-1}
\centering
\begin{minipage}[h]{0.48\textwidth}
\vspace{-5em}
\centering
\begin{subfigure}[h]{0.43\linewidth}
    \centering
    \includegraphics[width=0.9\linewidth]{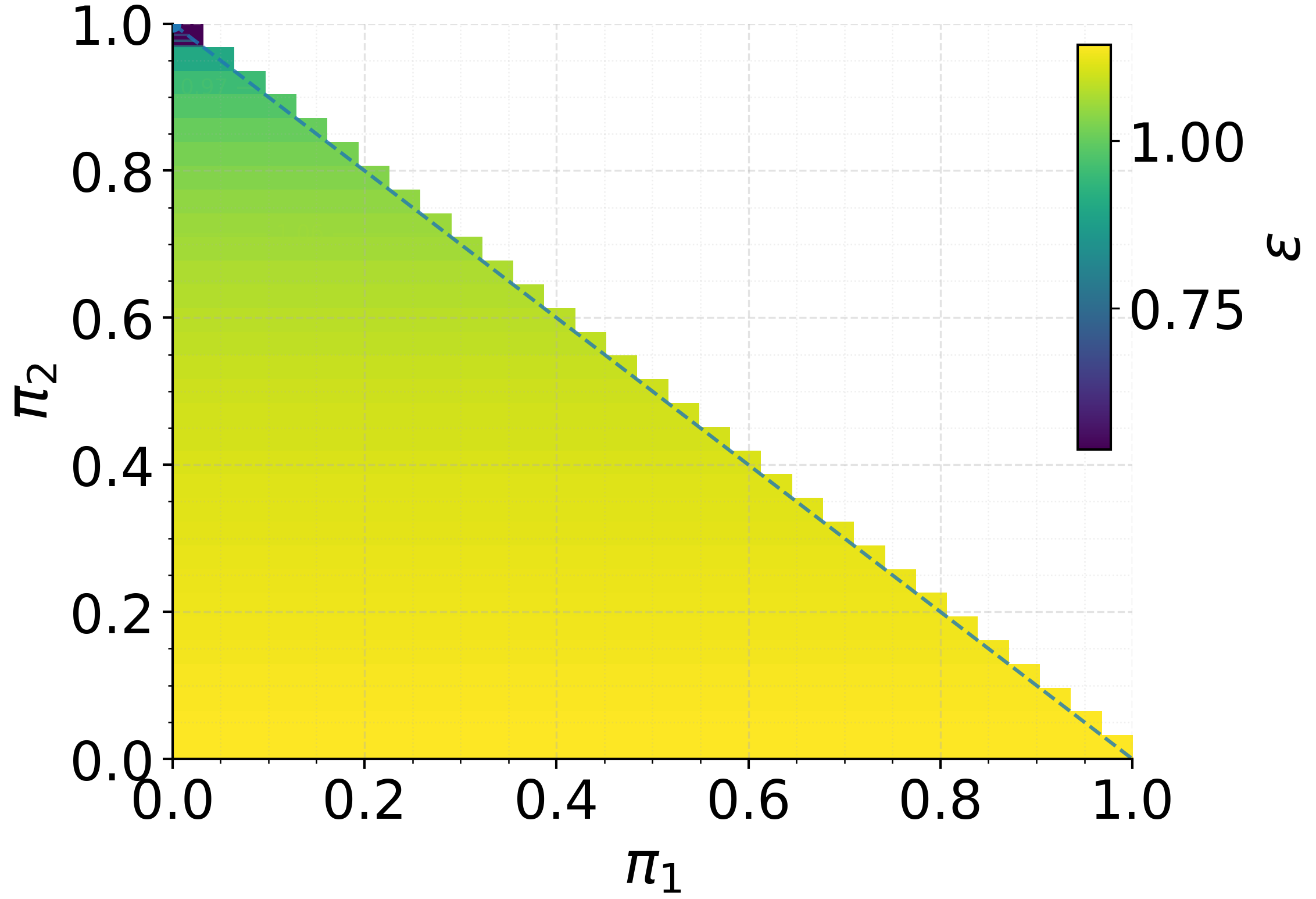}
    \caption{\small RS with RDP}
\end{subfigure}
\hfill
\begin{subfigure}[h]{0.43\linewidth}
    \centering
    \includegraphics[width=0.9\linewidth]{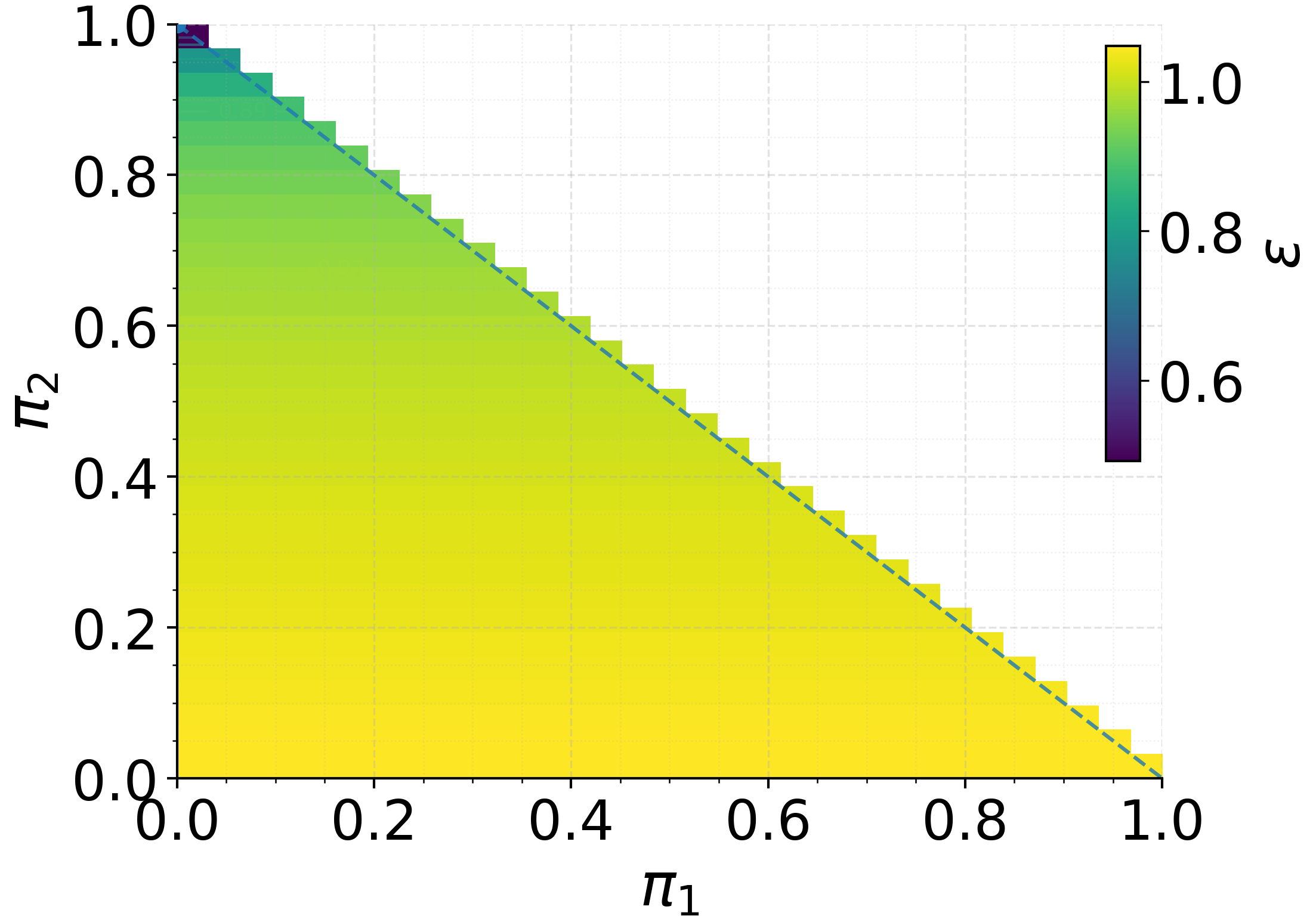}
    \caption{\small RS with PLD}
\end{subfigure}

\begin{subfigure}[h]{0.43\linewidth}
    \centering
    \includegraphics[width=0.9\linewidth]{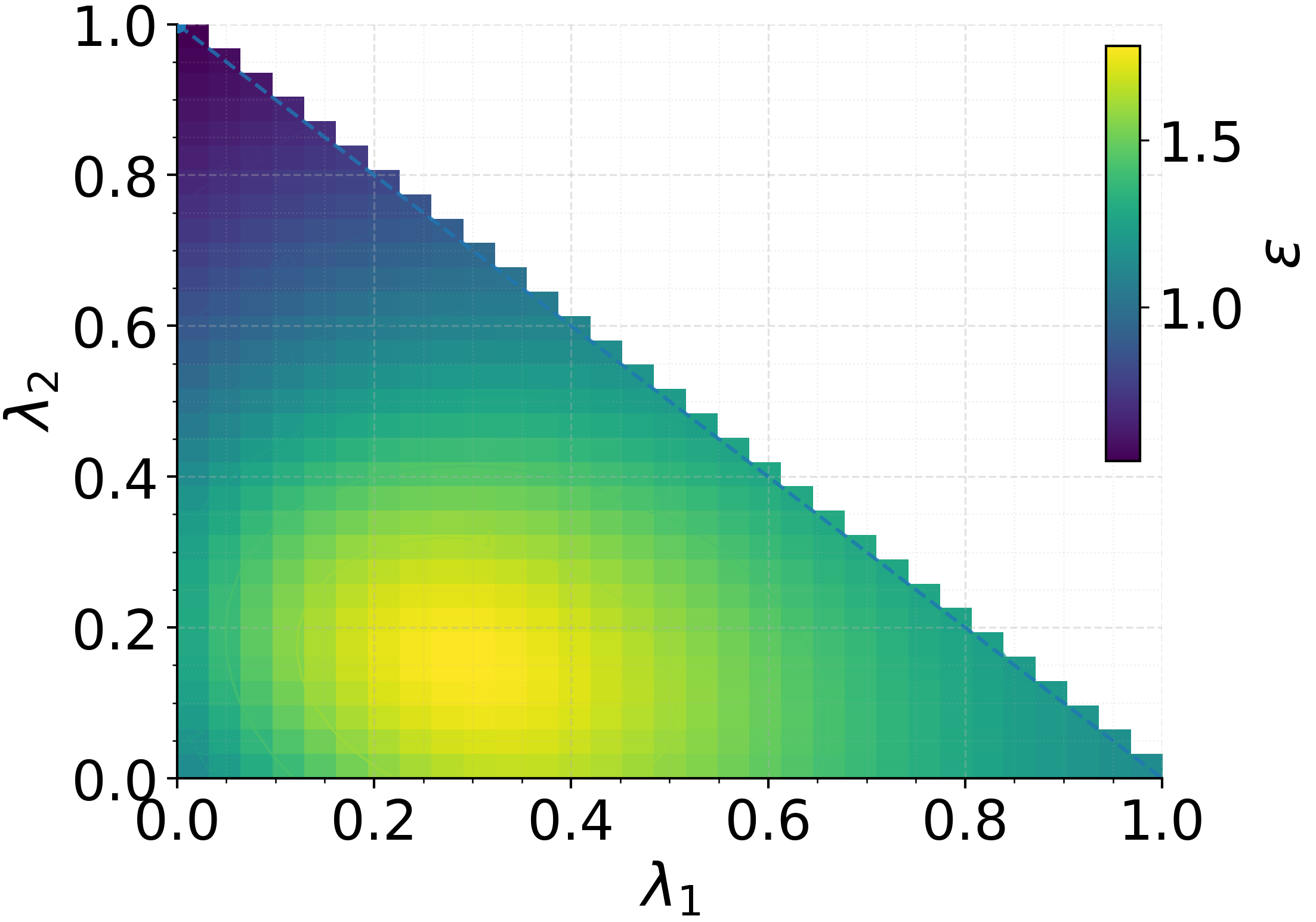}
    \caption{\small LC with RDP}
\end{subfigure}
\hfill
\begin{subfigure}[h]{0.43\linewidth}
    \centering
    \includegraphics[width=0.9\linewidth]{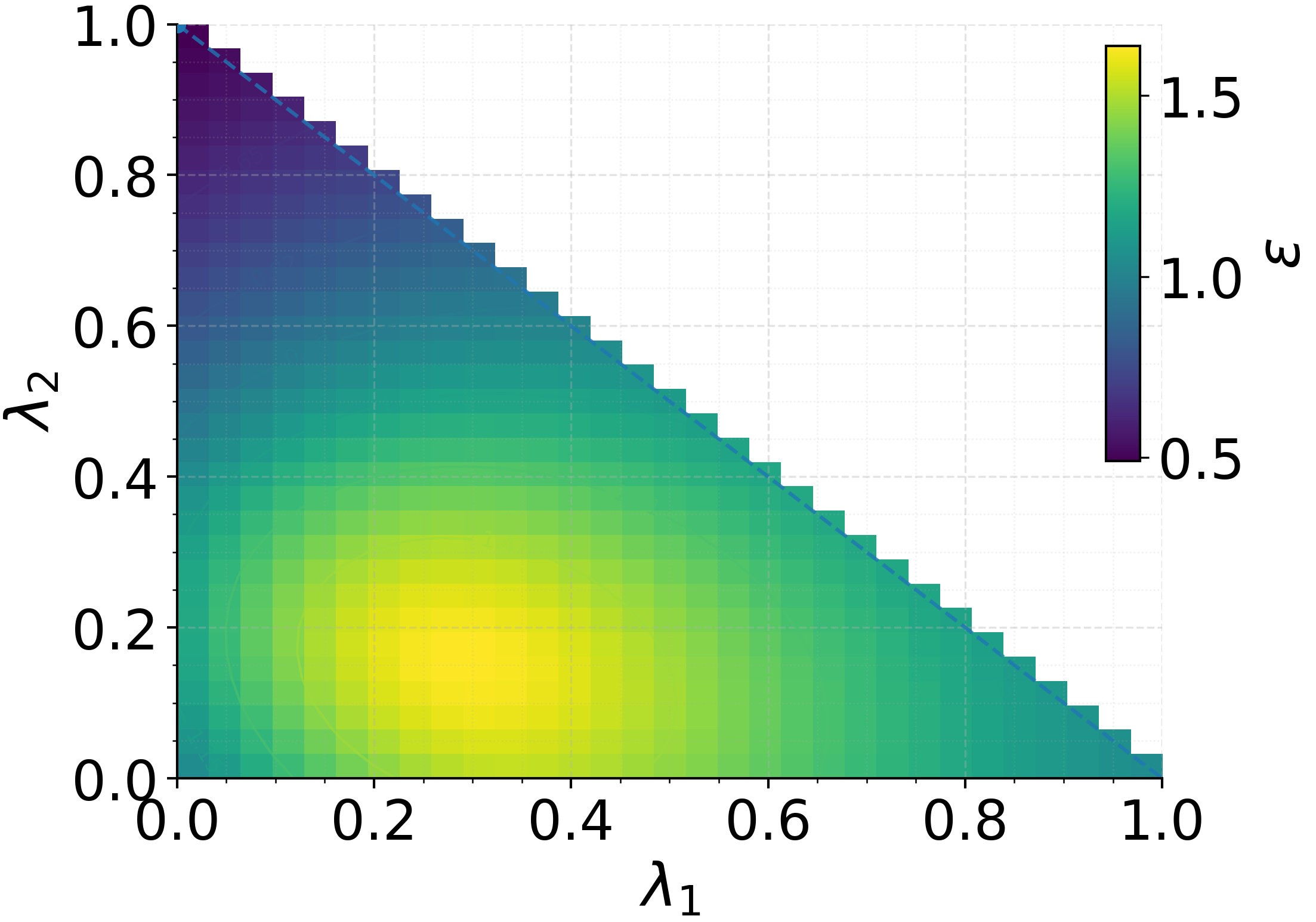}
    \caption{\small LC with PLD}
\end{subfigure}
\vspace{-0.03in}
\captionof{figure}{\small Heatmap of the privacy parameter $\varepsilon$ on MNIST ($\delta=10^{-5}$) as the merging weights vary when combining three models.}
\label{fig:Heatmap-MNIST}
\end{minipage}
\hfill
\begin{minipage}[t]{0.45\textwidth}
\centering
\includegraphics[width=\linewidth]{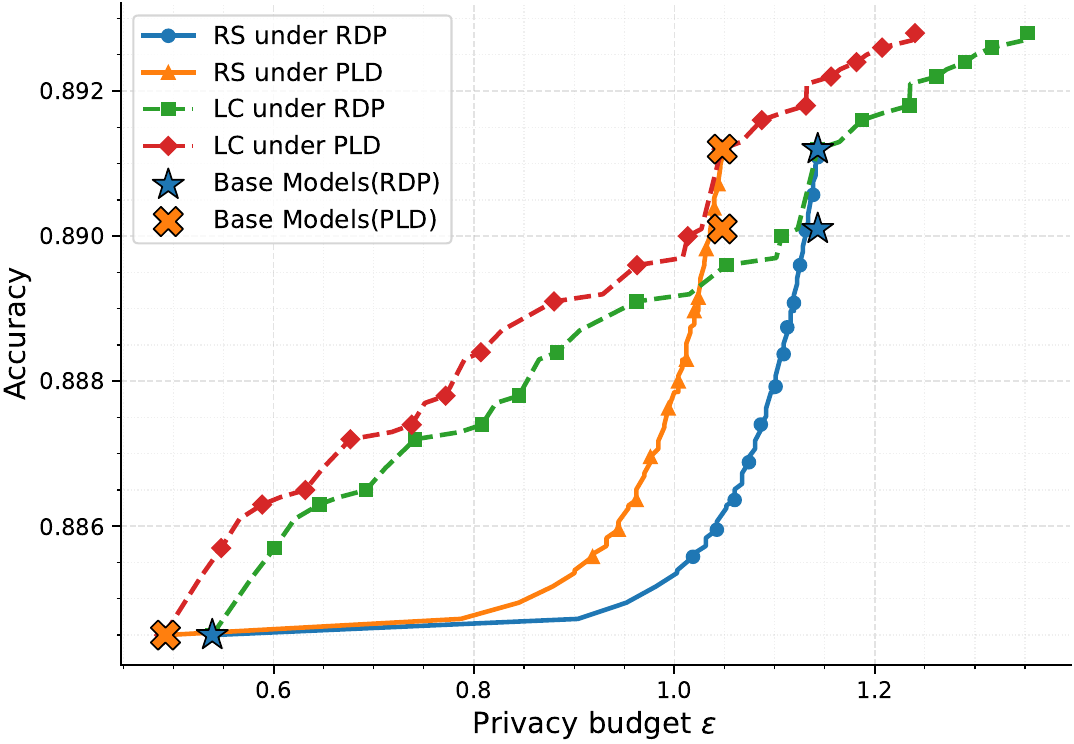}
\captionof{figure}{\small Privacy/utility tradeoffs of MNIST ($\delta=10^{-5}$) on a convex model. We see that PLD is tighter than RDP in practice, and LC generally achieves a better privacy/utility tradeoff than RS.}
\label{fig:pare-MNIST}
\end{minipage}
\vspace{-0.5em}
\end{figure}

Figure~\ref{fig:Heatmap-MNIST} reports the privacy changes as different mixing weights vary under both RS and LC with RDP and PLD.
Both methods successfully adapt to flexible privacy requirements. 
Figure~\ref{fig:pare-MNIST} shows the Pareto frontier of RS and LC with RDP and PLD on MNIST when combining three models. 
All methods exhibit a clear accuracy/privacy tradeoff. Across all settings, the $\varepsilon$ values guaranteed by PLD are consistently smaller than those obtained from RDP.

We next evaluate a nonconvex setting on CIFAR-10 using a pretraining-based pipeline. 
We randomly split the training set into two disjoint halves, $D_{\mathrm{pre}}$ and $D_{\mathrm{priv}}$. 
We first train a non-private ResNet18 model on $D_{\mathrm{pre}}$ using standard SGD, and then use this model as initialization for DP-SGD on $D_{\mathrm{priv}}$. 
Starting from this initialization, we train multiple private models with different DP-SGD hyperparameters, yielding different privacy/utility tradeoffs. 
We then apply RS and LC to these private models as in the MNIST experiment. 
The privacy guarantee in this experiment is with respect to $D_{\mathrm{priv}}$ only, since $D_{\mathrm{pre}}$ is used for non-private pretraining.

\begin{wrapfigure}[]{r}{0.45\textwidth} 
  \centering
  \vspace{-1em}
  \includegraphics[width=0.45\textwidth]{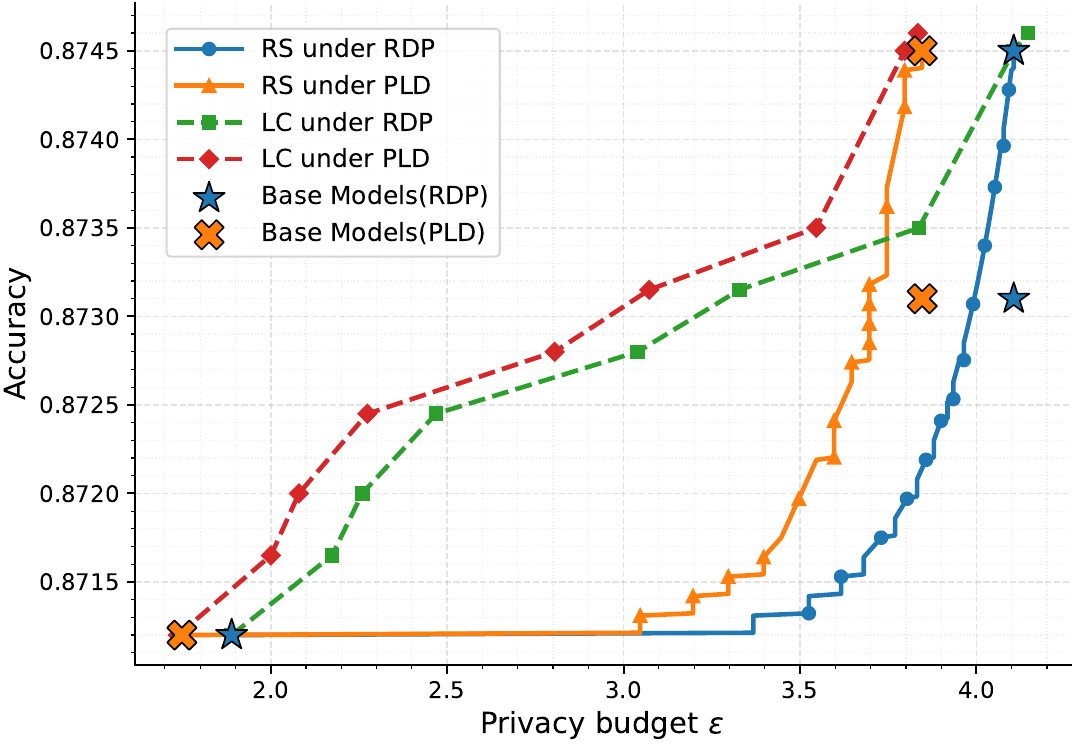}
  \captionof{figure}{\small Privacy/utility tradeoffs of model merging on CIFAR-10 (target $\delta=10^{-5}$). Input private models are finetuned from the same pretrained model.}
    \label{fig:pare-CIFAR-pretrain}
\vspace{-1em}
\end{wrapfigure}

Figure~\ref{fig:pare-CIFAR-pretrain} shows the Pareto frontiers of RS and LC under both RDP- and PLD-based accounting. 
Both methods enable flexible adaptation to different target privacy budgets without additional training, and PLD-based accounting remains consistently tighter than RDP-based accounting. 

Interestingly, in this pretraining-based setting, we also observe a phenomenon that was already present in some of our earlier experiments: the merged model may outperform all individual input models. This may be because a stronger common initialization places the input models in a more well-behaved region of the parameter space, making merging more effective at preserving shared structure while averaging out part of the DP-induced noise. Similar observations have also appeared in the non-private model merging literature \cite{wortsman2022model,matena2022merging}. This further highlights the value of our framework: model merging is not only a tool for adapting to different target privacy budgets after training, but may also provide utility gains relative to the input models.

\vspace{-0.15in}
\section{Conclusion, Limitation, and Future Directions} \label{sec: conclusion}
\vspace{-0.09in}
To the best of our knowledge, this is the first work that studies model merging to meet flexible privacy requirements during deployment time. We have proposed two merging strategies, based on random selection and linear combination, without any additional training steps. We provide principled privacy accounting based on both RDP and PLD, and our experiments demonstrate the practical accuracy/privacy
tradeoffs of the proposed methods across synthetic and real datasets and models.
Several directions remain open. One limitation of our current LC analysis is that it is developed for full-batch DP-SGD; extending the accounting to subsampled DP-SGD is an interesting direction for future work.
It would also be interesting to design and analyze
more sophisticated data-dependent merging algorithms that better handle challenging regimes such as nonconvex models. On the practical side, an important next step is to develop
effective and private procedures for tuning merging parameters (e.g., selecting $\lambda$ and $\pi$) under privacy constraints.

\bibliography{ref}
\bibliographystyle{plainnat}

\newpage
\appendix
\onecolumn

\section{Composition Theory and Algorithm for LC} \label{app: composition}
\subsection{Composition Theory} \label{app:composition theory}

We first recall the adaptive composition property of RDP, which justifies summing the conditional per-step RDP costs.

\begin{theorem}[Adaptive composition of RDP; Proposition 1 in \citep{mironov2017renyi}]
\label{prop:rdp-adaptive-composition}
Let $f:\mathcal D \to \mathcal Y_1$ be an $(\alpha,\varepsilon_1)$-RDP mechanism.
Let $g:\mathcal Y_1 \times \mathcal D \to \mathcal Y_2$ be such that, for every
$y_1\in\mathcal Y_1$, the mechanism $g(y_1,\cdot)$ satisfies
$(\alpha,\varepsilon_2)$-RDP. 
Then the adaptive composition mechanism
\[
h(D) = (Y_1,Y_2), 
\qquad Y_1\sim f(D), \quad Y_2\sim g(Y_1,D),
\]
satisfies $(\alpha,\varepsilon_1+\varepsilon_2)$-RDP.
\end{theorem}

Combining Theorem~\ref{thm:rdp-linmix} with Theorem~\ref{prop:rdp-adaptive-composition}, we obtain the composed RDP profile by summing the per-step costs and then convert it to an $(\varepsilon,\delta)$-DP guarantee using the standard RDP-to-DP conversion~\citep{mironov2017renyi}.

We also recall the standard PLD composition rule, which is used to compose the per-step PLDs.

\begin{theorem}[Composition of privacy loss distributions; Theorem 1 in \citep{sommer2018privacy}]
\label{thm:pld-composition}
For two independent mechanisms with privacy loss distributions $\omega$ and $\omega'$,
the privacy loss distribution of their joint release is given by convolution:
\[
\omega_c = \omega * \omega' .
\]
\end{theorem}

Combining with Theorem~\ref{thm:pld-linmix} and Corollary~\ref{cor:pld-linmix-compose}, we know that the PLD of the full mechanism can be upper bounded by convolving the practical surrogate PLDs.

\subsection{Algorithm for Setting Combination Coefficients for LC}

We summarize the accounting procedure for LC that output a set of coefficients $\lambda \in \Delta^{N-1}$ in Algorithm~\ref{algo:LC-combined} below.

\begin{algorithm}[h!]
\caption{Setting Combination Coefficients for Linear Combination (LC)}
\label{algo:LC-combined}
\begin{algorithmic}[1]

\State {\bfseries Input:} target privacy $(\varepsilon', \delta')$, orders grid $\mathcal{A}$ for RDP, 
       hyperparameter sets $\{\mathcal{S}_i\}_{i=1}^N$ of $\{\mathcal{M}_i\}_{i=1}^N$, total step $T$
\State {\bfseries Output:} set of combination coefficients $\lambda \in\Delta^{N-1}$ satisfying the target privacy.
\vspace{-0.07in}
\Statex \rule{\linewidth}{0.4pt}  
\Statex\noindent  
\begin{minipage}[t]{0.51\linewidth}
  \textit{Case 1: RDP accounting}
  \begin{algorithmic}[1]   
    \Function{DP\_Eps}{($\lambda,\delta$)}
    \For{$\alpha \in \mathcal{A}$}
    \For{$t=1$ to $T$}
        \State Compute the per-step RDP parameter $\varepsilon^{(t)}_{\alpha}(\lambda)$ via \eqref{ineq:perstep-LC}.
        \EndFor
    \State $\varepsilon_{\alpha}^{\text{LC}}(\lambda) \gets \sum_{t=1}^T\varepsilon^{(t)}_{\alpha}(\lambda)$
    \State $\varepsilon_{\alpha}(\lambda;\delta) \gets \varepsilon_{\alpha}^{\text{LC}}(\lambda) + \frac{\log(1/\delta)}{\alpha-1}$
    \EndFor
  \State {\bfseries return:} $\min_{\alpha\in\mathcal{A}} \varepsilon_{\alpha}(\lambda;\delta')$
\EndFunction

\State {\bfseries return} $\{\lambda \mid \text{DP\_Eps($\lambda,\delta'$)}\leq  \varepsilon^\prime\}$
  \end{algorithmic}
\end{minipage}%
\hfill%
\begin{minipage}[t]{0.49\linewidth}
  \textit{Case 2: PLD accounting}
  \begin{algorithmic}[1]
    \Function{DP\_Delta}{($\lambda, \varepsilon$)}
    \For{$t=1$ to $T$}
    \State Compute the per-step privacy-loss random variable $\tilde L^{(t)}$ via Eq. \eqref{eq:randomv-LC-surrogate} 
    \State Discretize $\tilde L^{(t)}$ on grid to obtain its PLD $\omega^{(t)}$
        \EndFor
\State Compose the $T$ steps via convolution 
    $
    \omega \gets \omega^{(1)} * \omega^{(2)} * \cdots * \omega^{(T)}
    $
\State Compute the induced privacy parameter $\delta^{\mathrm{LC}}(\varepsilon)$ from the composed PLD $\omega$.
    \State \textbf{return} $\delta^{\mathrm{LC}}(\varepsilon)$
\EndFunction

\State {\bfseries return} $\{\lambda \mid \text{DP\_Delta($\lambda,\varepsilon^\prime$)}\leq  \delta^\prime\}$
  \end{algorithmic}
\end{minipage}
\end{algorithmic}
\end{algorithm}

\section{Compare with Other Composition Methods}\label{sec:composition}
Extensive prior literature has studied DP composition of jointly releasing multiple models
\citep[e.g.,][]{dwork2010boosting,kairouz2015composition}. Among them, \citet{dwork2010boosting} is one of the seminal works, and \citet{kairouz2015composition} further improves the bounds. However, these results are
stated purely using the privacy  information in the form of $(\varepsilon_i,\delta_i)$ parameters of each model.
In contrast, assuming access to RDP and PLD parameters typically yields tighter
accounting. 
Hence, the joint-release bounds in Theorem~\ref{thm:linmix-profile-bound} (the right-hand sides of Eq~\eqref{eq:linmix-rdp-sum} and Eq.~\eqref{eq:linmix-pld-upper}) using RDP/PLD parameters are in general
tighter than $(\varepsilon,\delta)$-only composition bounds. 
Besides, beyond simply assuming access to RDP and PLD parameter, our analysis makes explicit how RS and LC interact with the underlying privacy
accounting. RS exploits mixture structure through log-sum-exp and convexity, while LC admits tight bounds when additional algorithmic structure is available. This
highlights why our bounds can improve over general joint-release accounting.
Compare to joint-release bounds in Theorem~\ref{thm:linmix-profile-bound}, our LC RDP/PLD accounting gives a tighter result than the joint-release RDP/PLD bound; and RS can also be tighter. 
Overall,
when richer privacy profiles (e.g., RDP parameters across orders) are available, our accounting procedures provide
stronger guarantees.  
Formally, we present Proposition~\ref{prop:rdp-vs-drv10-gaussian} below, which establishes our argument for the case of $N$ independent and identical Gaussian mechanisms. See complete proofs in Appendix~\ref{app:proof:tighter}. 
\begin{proposition}
\label{prop:rdp-vs-drv10-gaussian}
Consider releasing $N$ models on the same dataset, where each release follows the same Gaussian mechanism
(e.g., mean estimation with sensitivity $\Delta$ and noise $Z\sim\mathcal{N}(0,\sigma^2)$).
Fix $\delta\in(0,\frac{1}{2})$. By RDP accounting, each model satisfies $(\varepsilon,\delta)$-DP.
Composing the $N$ releases using the $(\varepsilon,\delta)$ advanced composition bound
\citep[\textnormal{Theorem~III.3}]{dwork2010boosting} with $\delta' = N\delta+\delta_0$ yields
a composed privacy parameter $\varepsilon_{\mathrm{com}}$.
Using the joint RDP bound (RHS of Eq.  (\ref{eq:linmix-rdp-sum})) 
in Theorem~\ref{thm:linmix-profile-bound} yields a guarantee
$(\varepsilon_{\mathrm{RDP}},\delta')$. We have that $
\varepsilon_{\mathrm{RDP}} \le \varepsilon_{\mathrm{com}}.$
\end{proposition}

\section{Complete Proofs} \label{app: proofs}
\subsection{Proof of Proposition \ref{prop:impossibility}}
\label{app:impossibility}
\begin{proof}
Consider following counterexample: $\theta_1(D) \equiv C$ with probability 1 and then $h(\theta_1,\theta_2)$ can be reduce to a function $\tilde h(\theta_2)$ that only depends on $\theta_2$. Since $h$ is nontrival, there at least are two points $\{a,b\}$ s.t $\tilde h(a) \stackrel{d}{\neq} \tilde h(b)$. Therefore there is at least one $S\subseteq\mathbb R^P$, such that 
$$
\bP(\tilde h(a)\in S) > \bP( \tilde h(b)\in S)
$$
Consider $\bP(\theta_2(D)=a)=p,\ \bP(\theta_2(D)=b)=1-p$, and $\bP(\theta_2(D^\prime)=a)=p^\prime,\ \bP(\theta_2(D)=b)=1-p^\prime$ such that 
$p=\frac{e^{\varepsilon_2}+\delta}{e^{\varepsilon_2}+1},  p^\prime=\frac{1-\delta}{e^{\varepsilon_2}+1}$. Thus $p=e^{\varepsilon_2}p^\prime+\delta$ and $1-p^\prime=e^{\varepsilon_2}(1-p)+\delta$
then
$$
\begin{aligned}
    \bP(\tilde h(\theta_2(D))\in S) &= p \bP(\tilde h(a)\in S)+ (1-p)\bP(\tilde h(b)\in S) \\
    &=e^{\varepsilon_2}p^\prime \bP(\tilde h(a)\in S) +\delta \bP(\tilde h(a)\in S)+e^{\varepsilon^\prime}(1-p^\prime)\bP(\tilde h(b)\in S)
    \\
    &\ \ \ \ -(e^{\varepsilon^\prime}-1+(e^{\varepsilon_2}-e^{\varepsilon^\prime})p^\prime+\delta)\bP(\tilde h(b)\in S) \\
    &>e^{\varepsilon^\prime}\bP(\tilde h(\theta_2(D^\prime))\in S)+((e^{\varepsilon_2}-e^{\varepsilon^\prime})p^\prime+\delta)\bP(\tilde h(a)\in S)\\
    &\ \ \ \ -(e^{\varepsilon^\prime}-1+(e^{\varepsilon_2}-e^{\varepsilon^\prime})p^\prime+\delta)\bP(\tilde h(b)\in S) \\
    &=e^{\varepsilon^\prime}\bP(\tilde h(\theta_2(D^\prime))\in S)+\frac{e^{\varepsilon_2}-e^{\varepsilon^\prime}+\delta e^{\varepsilon^\prime}+\delta}{e^{\varepsilon_2}+1}(\bP(\tilde h(a)\in S)-\bP(\tilde h(b)\in S))\\
    &\ \ \ \ -(e^{\varepsilon^\prime}-1)\bP(\tilde h(b)\in S) 
\end{aligned}
$$

The last equality is gained by bringing $p^\prime=e^{-\varepsilon_2}(1-\delta)$ into the equation. 

If $D_{TV}(\tilde h(a),\tilde h(b))>\frac{(\delta^\prime+e^{\varepsilon^\prime}-1)(e^{\varepsilon_2}+1)}{e^{\varepsilon_2+\varepsilon^\prime}+\delta e^{\varepsilon^\prime}+\delta-1}$, we know 
$$
\begin{aligned}
    &\sup_{S} \left(\frac{e^{\varepsilon_2}-e^{\varepsilon^\prime}+\delta e^{\varepsilon^\prime}+\delta}{e^{\varepsilon_2}+1}(\bP(\tilde h(a)\in S)-\bP(\tilde h(b)\in S))  -(e^{\varepsilon^\prime}-1)\bP(\tilde h(b)\in S) \right)\\
    &\geq \frac{e^{\varepsilon_2}-e^{\varepsilon^\prime}+\delta e^{\varepsilon^\prime}+\delta}{e^{\varepsilon_2}+1}D_{TV}(\tilde h(a),\tilde h(b))-(e^{\varepsilon^\prime}-1)(1-D_{TV}(\tilde h(a),\tilde h(b)))\\
    &=  \frac{e^{\varepsilon_2+\varepsilon^\prime}+\delta e^{\varepsilon^\prime}+\delta-1}{e^{\varepsilon_2}+1}D_{TV}(\tilde h(a),\tilde h(b))-e^{\varepsilon^\prime}+1 \\
    &>\delta^\prime
\end{aligned}
$$
and thus there exists $S$ such that
$$
\begin{aligned}
    \bP(\tilde h(\theta_2(D))\in S) > 
    e^{\varepsilon^\prime}\bP(\tilde h(\theta_2(D^\prime))\in S)+\delta^\prime.
\end{aligned}
$$
We finish the proof.
\end{proof}

\subsection{Proof of Lemma \ref{lem:RDP-mean-estimation} and Theorem \ref{thm:LC-dominates-alpha}}

\begin{lemma}
\label{lem:RDP-mean-estimation}
For any $\pi,\lambda\in [0,1]$, we have RDP parameter $(\alpha,\varepsilon_{\alpha}^{\text{RS}}(\pi))$ for RS satisfying
$$
\varepsilon_{\alpha}^{\text{RS}}=\mathcal{D}_\alpha\big(P_\pi\|Q_\pi\big)
>
\frac{\alpha\Delta^2}{2\sigma_{RS}^2(\pi)}+c_{\pi}\Delta^2+o(\Delta^2),
$$
when $\Delta\to 0$. Here $c_{\pi} \ge 0$ is a constant depending on $\pi$, and when $\pi\in (0,1)$, $c_{\pi} > 0$. And RDP parameter $(\alpha,\varepsilon^{\text{LC}}_{\alpha}(\lambda))$  for LC satisfying
$
\varepsilon^{\text{LC}}_{\alpha}
=\frac{\alpha\Delta^2}{2\sigma_{LC}^2(\lambda)}.
$
\end{lemma}

\begin{proof}[Proof of Lemma \ref{lem:RDP-mean-estimation}]
Since under linear combination with parameter $\lambda$, the output $\theta(D)\sim \mathcal{N}(\hat \mu, \lambda^2\sigma_1^2+(1-\lambda)^2\sigma_2^2)$ and $\sigma_{LC}^2(\lambda)=\lambda^2\sigma_1^2+(1-\lambda)^2\sigma_2^2$. 
Thus
$$
\varepsilon^{\text{LC}}_{\alpha}
 =\frac{\alpha\Delta^2}{2\sigma_{LC}^2(\lambda)}.
$$

Next we prove 
$$
\varepsilon_{\alpha}^{\text{RS}}=\mathcal{D}_\alpha\big(P_\pi\|Q_\pi\big)
\ge
\frac{\alpha\Delta^2}{2\sigma_{RS}^2(\pi)}+o(\Delta^2). 
$$
Write $p$ for the density of $P_\pi$ and note that $Q_\pi$'s density \ $q(x)=p(x-\Delta)$. Define
\[
J_\chi(P)\;\coloneqq\;\int_{\mathbb R}\frac{(p'(x))^2}{p(x)}\,dx.
\]

\textbf{Step 1.}
We show that
\begin{equation}\label{eq:disp}
\mathcal{D}_\alpha(P_\pi\|Q_\pi)\;\ge\;\frac{\alpha}{2}\,J_\chi(P_\pi)\,\Delta^2+o(\Delta^2).
\end{equation}

Consider $\Psi_\alpha(\Delta)\coloneqq \log\!\int p(x)^\alpha\,p(x-\Delta)^{\,1-\alpha}\,dx$, so that
$\mathcal{D}_\alpha(P_\pi\|Q_\pi)=\Psi_\alpha(\Delta)/(\alpha-1)$. Differentiate over $\Delta$ and note that $\int p^\prime dx=0$,
we know 
$\Psi_\alpha'(0)=0$ and
$\Psi_\alpha''(0)=\alpha(\alpha-1)\,J_\chi(P)$, which yields \eqref{eq:disp}.

\textbf{Step 2.}
Let random variable $X \sim P_{\pi}$, we next show
\begin{equation}\label{eq:Jchi-var}
J_\chi(P_{\pi})\;\geq \;\frac{1}{\operatorname{Var}(X)}.
\end{equation}
Let $s(x)=p'(x)/p(x)$ be the score. For any smooth $h$,
integration by parts gives $\mathbb E[h'(X)]=-\mathbb E[h(X)s(X)]$.
By Cauchy-Schwarz,
$(\mathbb E[h'(X)])^2\le \mathbb E[h(X)^2]\;\mathbb E[s(X)^2]=\mathbb E[h(X)^2]\;J_\chi(P)$.
Taking the supremum over $h$ shows
$J_\chi(P)\ge \sup_h \big(\mathbb E[h'(X)]^2/\mathbb E[h(X)^2]\big)$.
Choosing $h(x)=x-\mathbb E[X]$ yields $\mathbb E[h'(X)]=1$ and
$\mathbb E[h(X)^2]=\operatorname{Var}(X)$, proving \eqref{eq:Jchi-var}. Note that when $\pi \in (0,1)$, $P_{\pi}$ is mixed gaussian distribution, thus the bound could not be reached and thus $c_\pi>0$. 

Combining step 1 and step 2 we finish the proof.
\end{proof}

\begin{proof}[Proof of Theorem~\ref{thm:LC-dominates-alpha}]
Fix $\alpha>1$ and $\pi\in[0,1]$, and let $\lambda$ satisfy
$\sigma_{\mathrm{LC}}^2(\lambda)=\sigma_{\mathrm{RS}}^2(\pi)$.
By Lemma~\ref{lem:RDP-mean-estimation},
\[
\varepsilon_{\alpha}^{\mathrm{LC}}(\lambda)
=
\frac{\alpha\Delta^2}{2\sigma_{\mathrm{LC}}^2(\lambda)}
=
\frac{\alpha\Delta^2}{2\sigma_{\mathrm{RS}}^2(\pi)}.
\]
Moreover,
\[
\varepsilon_{\alpha}^{\mathrm{RS}}(\pi)
>
\frac{\alpha\Delta^2}{2\sigma_{\mathrm{RS}}^2(\pi)}
+c_\pi\Delta^2+o(\Delta^2).
\]
If $\pi\in(0,1)$, then $c_\pi>0$, so there exists
$C_{\pi,\alpha}>0$ such that, for all $0<\Delta<C_{\pi,\alpha}$,
\[
\varepsilon_{\alpha}^{\mathrm{RS}}(\pi)
>
\varepsilon_{\alpha}^{\mathrm{LC}}(\lambda).
\]
For $\pi\in\{0,1\}$, RS reduces to a single Gaussian mechanism, and the two
RDP costs are equal whenever
$\sigma_{\mathrm{LC}}^2(\lambda)=\sigma_{\mathrm{RS}}^2(\pi)$.
Thus LC achieves privacy no worse than RS.

Finally, choosing such a $\lambda$ gives
$\sigma_{\mathrm{LC}}^2(\lambda)=\sigma_{\mathrm{RS}}^2(\pi)$, so LC has
the same MSE as RS and no worse order-$\alpha$ RDP privacy. Hence, under the
same target RDP budget, LC achieves no larger MSE.
\end{proof}

\subsection{Proof of Theorem \ref{prop:rdp-mixture}} \label{app:proof:rs_rdp}
\begin{lemma}\label{lem:weighted-holder}
For any $\alpha>1$, $\pi\in \Delta^{N-1}$, and $a_i,b_i\ge 0$ ($i=1,2,\ldots,N$),
\[
\big(\sum_{i=1}^N\pi_i a_i\big)^{\alpha}
\big(\sum_{i=1}^N\pi_i b_i\big)^{1-\alpha}
\;\le\;
\sum_{i=1}^N\pi_i a_i^{\alpha} b_i^{\,1-\alpha}.
\]
\end{lemma}

\begin{proof}[Proof of \ref{lem:weighted-holder}]
    Let the conjugate exponents be $(r,s)=(\alpha,\alpha/(\alpha-1))$.
Let $u_i=a_i\,b_i^{(1-\alpha)/\alpha}$, $v_i=b_i^{(\alpha-1)/\alpha}$.
By Hölder’s inequality, we have
\[
\sum_i \pi_i a_i
=\sum_i \pi_i u_i v_i
\le \Big(\sum_i \pi_i u_i^r\Big)^{1/r}\!\Big(\sum_i \pi_i v_i^s\Big)^{1/s}
=
\Big(\sum_i \pi_i a_i^\alpha b_i^{1-\alpha}\Big)^{1/\alpha}
\Big(\sum_i \pi_i b_i\Big)^{(\alpha-1)/\alpha}.
\]
Raising both sides to the power $\alpha$ and we finish the proof.\qed
\end{proof}

\begin{proof}[Proof of $\varepsilon_{\alpha}^{RS}(\pi)$ upper bound]
Let $P_i=\mathsf{Law}(\mathcal{M}_i(D))$ and $Q_i=\mathsf{Law}(\mathcal{M}_i(D^\prime))$ be the output distribution of the two DP mechanisms on neighboring datasets $D\sim D^\prime$.
Let
\[
P_{\mathrm{mix}}=\sum_{i=1}^N \pi_i P_i,\qquad
Q_{\mathrm{mix}}=\sum_{i=1}^N \pi_i Q_i,
\]
be distributions with densities $p_{\mathrm{mix}}=\sum_{i=1}^N \pi_i p_i$ and $q_{\mathrm{mix}}=\sum_{i=1}^N \pi_i q_i$.
By the definition,
\[
e^{(\alpha-1)\mathcal{D}_\alpha(P_{\mathrm{mix}}\|Q_{\mathrm{mix}})}
=\int \! p_{\mathrm{mix}}^\alpha\, q_{\mathrm{mix}}^{\,1-\alpha}\,dx
=\int \!\Big(\sum_{i=1}^N \pi_i p_i\Big)^\alpha
         \Big(\sum_{i=1}^N \pi_i q_i\Big)^{1-\alpha}\!dx.
\]
Applying Lemma~\ref{lem:weighted-holder}, we know 
\[
p_{\mathrm{mix}}^\alpha(x)\, q_{\mathrm{mix}}^{\,1-\alpha}(x)
\le
\, \sum_{i=1}^N \pi_ip_i^\alpha(x)\, q_i^{\,1-\alpha}(x).
\]
Thus,
\[
e^{(\alpha-1)\mathcal{D}_\alpha(P_{\mathrm{mix}}\|Q_{\mathrm{mix}})}
\le
\sum_{i=1}^N \pi_i \!\int p_i^\alpha q_i^{\,1-\alpha}
=
\sum_{i=1}^N \pi_i e^{(\alpha-1)\mathcal{D}_\alpha(P_i\|Q_i)}.
\]
Thus,
\[
\mathcal{D}_\alpha(P_{\mathrm{mix}}\|Q_{\mathrm{mix}})
\le
\frac{1}{\alpha-1}\log\!\Big(
\sum_{i=1}^N \pi_i\, e^{(\alpha-1)\mathcal{D}_\alpha(P_i\|Q_i)}
\Big).
\]
Since 
$\varepsilon_{\alpha,i} = \sup_{D\sim D^\prime} \mathcal{D}_\alpha(P_i\|Q_i)$, we obtain
\[
\varepsilon_\alpha^{\text{RS}}(\pi)
\le
\frac{1}{\alpha-1}\log\!\Big(
\sum_{i=1}^N \pi_i\,e^{(\alpha-1)\varepsilon_{\alpha,i}} 
\Big),\quad \forall \alpha>1.
\]
We finish the proof.
\end{proof}

\subsection{Proof of Theorem \ref{prop:pld-convex-mixture}} \label{app:proof:rs_pld}

\begin{proof}[Proof of Theorem \ref{prop:pld-convex-mixture}]
For any $\varepsilon\ge 0$, recall the hockey-stick divergence is defined as
\begin{equation}\label{eq:hockey_supA}
\mathcal H_\varepsilon(P,Q)
=\sup_{S}\Big\{P(S)-e^\varepsilon Q(S)\Big\},
\end{equation}

Thus for any measurable $S$,
\[
P_{RS}(S)-e^\varepsilon Q_{RS}(S)
=\sum_{i=1}^N \pi_i\Big(P_i(S)-e^\varepsilon Q_i(S)\Big).
\]
Taking supremum over $S$,
we obtain
\[
\mathcal H_\varepsilon(P_{RS},Q_{RS})
=\sup_A \sum_{i=1}^N \pi_i\Big(P_i(A)-e^\varepsilon Q_i(A)\Big)
\le
\sum_{i=1}^N \pi_i \sup_A \Big(P_i(A)-e^\varepsilon Q_i(A)\Big)
=
\sum_{i=1}^N \pi_i \mathcal H_\varepsilon(P_i,Q_i).
\]
This proves the first inequality. The second inequality follows by the same argument after swapping the roles of $P$ and $Q$:
\[
\mathcal H_\varepsilon(Q_{RS},P_{RS})
\le
\sum_{i=1}^N \pi_i \mathcal H_\varepsilon(Q_i,P_i).
\]

Using the two bounds above for each neighboring pair $D\sim D'$, we have
\[
\delta^{\text{RS}}(\varepsilon)
\le
\sup_{D\sim D'}\max\Big\{\sum_{i=1}^N\pi_i\,\mathcal H_\varepsilon(P_i,Q_i),\
\sum_{i=1}^N\pi_i\,\mathcal H_\varepsilon(Q_i,P_i)\Big\},
\]
We finish the proof.
\end{proof}

\subsection{Proof of Theorem \ref{thm:linmix-profile-bound}} \label{app:proof:lc_worst_case}
\begin{proof}
Fix any neighboring datasets $D\sim D'$, let $P^{\mathrm{joint}},Q^{\mathrm{joint}}$ be the
output distributions of the joint release
\(
\mathcal M_{\mathrm{joint}}:=(\mathcal M_1,\ldots,\mathcal M_N)
\)
on $D$ and $D'$, respectively. Similarly, we define $P^{\text{LC}},Q^{\text{LC}}$ be the
output distributions of LC on $D$ and $D'$.
Obviously, the linear-combination mechanism can be understood as a post-processing of the joint release.

\paragraph{(1) RDP upper bound.}
Since the linear-combination mechanism can be understood as a post-processing of the joint release, we have
\begin{equation}\label{eq:rdp-dpi}
D_\alpha(P^{\text{LC}}\,\|\,Q^{\text{LC}})
\le D_\alpha(P^{\mathrm{joint}}\,\|\,Q^{\mathrm{joint}}).
\end{equation}
Under the definition of Rényi divergence, we have
\begin{equation}\label{eq:rdp-comp}
D_\alpha(P^{\mathrm{joint}}\,\|\,Q^{\mathrm{joint}})
= \sum_{i=1}^N D_\alpha(P_i\,\|\,Q_i)
\leq \sum_{i=1}^N \varepsilon_{\alpha,i}.
\end{equation}
Combining \eqref{eq:rdp-dpi}-\eqref{eq:rdp-comp} and taking the worst case over $D\sim D'$ yields
\(
\varepsilon^{\text{LC}}_\alpha(\lambda)\le \sum_{i=1}^N \varepsilon_{\alpha,i}.
\)

\paragraph{(2) PLD upper bound.}
Since the linear-combination mechanism can be understood as a post-processing of the joint release, we have 
\begin{equation}\label{eq:hockey-dpi}
\mathcal H_\varepsilon(P^{\text{LC}}\,\|\,Q^{\text{LC}})\le \mathcal H_\varepsilon(P^{\mathrm{joint}}\,\|\,Q^{\mathrm{joint}}).
\end{equation}
Similarly, we have
\[
\mathcal H_\varepsilon(Q^{\text{LC}},P^{\text{LC}})\le \mathcal H_\varepsilon(Q^{\mathrm{joint}},P^{\mathrm{joint}}).
\]
Therefore,
\[
\delta^{\text{LC}}(\varepsilon)
=\max\{\mathcal H_\varepsilon(P^{\text{LC}},Q^{\text{LC}}),\mathcal H_\varepsilon(Q^{\text{LC}},P^{\text{LC}})\}
\le
\max\{\mathcal H_\varepsilon(P^{\mathrm{joint}},Q^{\mathrm{joint}}),\mathcal H_\varepsilon(Q^{\mathrm{joint}},P^{\mathrm{joint}})\}
=: \delta_{\mathrm{joint}}(\varepsilon).
\]

Moreover, from PLD accounting, we know
the PLD of the joint release is the convolution
\(
\omega_{\mathrm{joint}}=\omega_1\star\cdots\star\omega_N,
\)
and $\delta_{\mathrm{joint}}(\varepsilon)$ is computed from $\omega_{\mathrm{joint}}$.

\paragraph{(3) Tightness.}
We show the bounds in \eqref{eq:linmix-rdp-sum}-\eqref{eq:linmix-pld-upper} are worst-case tight by an explicit construction.
Let $p_1,\ldots,p_N$ be positive integers and set $p=\sum_{i=1}^N p_i$.
For each $i\in[N]$, let $\mathcal M_i^\prime$ be any mechanism whose output lies in $\mathbb R^{p_i}$, and $\mathcal M_i$ output its p dimensional version $\theta_i:\mathcal D\to\mathbb R^p$, where
the $\sum_{j=1}^{i-1}p_j+1-$th element to $\sum_{j=1}^{i}p_j-$th element is $M_i^\prime(D)$ and other elements are 0.

Now consider the linear-combination output
\[
\theta(D)\;:=\;\sum_{i=1}^N \lambda_i\theta_i(D)\in\mathbb R^p.
\]
Intuitively, linear-combination does not lose any information from the joint release. As a result, for every neighboring pair $D\sim D'$ and every $\alpha>1$,
Rényi divergence is preserved. We have
\[
D_\alpha\!\big(P^{\text{LC}}\,\|\,Q^{\text{LC}}\big)
= D_\alpha\!\big(P_{\mathrm{joint}}\,\|\,Q_{\mathrm{joint}}\big),
\]
and likewise the hockey-stick divergence (hence the PLD profile) is also preserved:
\[
\delta^{\text{LC}}(\varepsilon)=\delta_{\mathrm{joint}}(\varepsilon),\qquad \forall\,\varepsilon\ge 0.
\]
Therefore the upper bounds \eqref{eq:linmix-rdp-sum}-\eqref{eq:linmix-pld-upper} cannot be improved in general.
 This establishes worst-case tightness.
\end{proof}

\subsection{Proof of Theorem \ref{thm:rdp-linmix}}
\label{app:proof:rdp-linmix}
\begin{proof}[Proof of Theorem~\ref{thm:rdp-linmix}]
We first establish a uniform conditional RDP bound for one step.
Let $D,D'$ be neighboring datasets under replace-one adjacency. 
Let $H$ and $H'$ denote arbitrary possible value of $(\theta_1^{(t-1)},\cdots, \theta_N^{(t-1)})$ before step $t$ under $D$ and $D'$. \[
\theta^{(t)}-\theta^{(t-1)}
=
\sum_{i=1}^N
\lambda_i\eta_i^{(t)} \widetilde g_i^{(t)}(D;H)
+
\overline e_\lambda^{(t)},
\]
where $\widetilde g_i^{(t)}(D;H)$ is clipped gradient for model $i$ and
\[
\overline e_\lambda^{(t)}
=
\sum_{i=1}^N \lambda_i\eta_i^{(t)} e_i^{(t)}
\sim
\mathcal N(0,(s_\lambda^{(t)})^2 I),
\qquad
s_\lambda^2
=
\sum_{i=1}^N
(\lambda_i\eta_i^{(t)}\sigma_i^{(t)} C_i^{(t)})^2 .
\]
Thus, conditional on $H$, the distribution of $\theta^{(t)}-\theta^{(t-1)}$ under $D$ is Gaussian with covariance $(s_\lambda^{(t)})^2 I$. Similarly, conditional on $H'$, the distribution under $D'$ is Gaussian with the same covariance.

By clipping,
\[
\left\|
\sum_{i=1}^N
\lambda_i\eta_i^{(t)}
\left(
\widetilde g_i^{(t)}(D;H)-\widetilde g_i^{(t)}(D';H')
\right)
\right\|_2
\le
2\sum_{i=1}^N \lambda_i\eta_i^{(t)} C_i^{(t)}
=
\Delta_\lambda ^{(t)}.
\]
Hence, for any pair of $H,H'$, the RDP between the two corresponding Gaussian conditional distributions is bounded by
\[
D_\alpha
\left(
\mathcal N(\sum_{i=1}^N\lambda_i\eta_i^{(t)} \widetilde g_i^{(t)}(D),(s_\lambda^{(t)})^2 I)
\middle\|
\mathcal N(\sum_{i=1}^N\lambda_i\eta_i^{(t)} \widetilde g_i^{(t)}(D'),(s_\lambda^{(t)})^2 I)
\right)
\le
\frac{\alpha(\Delta_\lambda^{(t)})^2}{2(s_\lambda^{(t)})^2}.
\]

Recall
\[
\mathcal F_{t-1}
=
\sigma(\theta^{(t-1)}).
\]
 The distribution of $\theta^{(t)}-\theta^{(t-1)}$ conditional on $\mathcal F_{t-1}$ under $D$ can be written as a mixture over $H$:
\[
\mathbb P_{\theta^{(t)}-\theta^{(t-1)}}(\cdot\mid D,\mathcal F_{t-1})
=
\int \mathbb P_{\theta^{(t)}-\theta^{(t-1)}}(\cdot \mid D,H)\,d\kappa(H\mid D, F_{t-1}),
\]
and similarly,
\[
\mathbb P_{\theta^{(t)}-\theta^{(t-1)}}(\cdot\mid D',\mathcal F_{t-1})
=
\int \mathbb P_{\theta^{(t)}-\theta^{(t-1)}}(\cdot \mid D,H')\,d\kappa(H'\mid D', F_{t-1}),
\]
where  $\kappa(\cdot\mid D,\mathcal F_{t-1})$ denotes the regular conditional distribution of $H$ and $H'$
 Similar to Theorem \ref{prop:rdp-mixture}, using joint convexity RDP, let $P, Q$ be the distribution of $\theta^{(t)}-\theta^{(t-1)}$ conditional on $\mathcal F_{t-1}$ under $D$ and $D'$ respectively, and
we obtain
\[
D_\alpha
\left(
P
\middle\|
Q
\right)
\le
\frac{\alpha(\Delta_\lambda^{(t)})^2}
{2(s_\lambda^{(t)})^2}.
\]
\end{proof}

\subsection{Proof of Theorem \ref{thm:pld-linmix} and Corollary \ref{cor:pld-linmix-compose}} \label{app:proof:pld-linmix}
\begin{proof}[Proof of Theorem~\ref{thm:pld-linmix}] 
Let $D,D'$ be neighboring datasets under replace-one adjacency. 
Let $H$ and $H'$ denote arbitrary possible value of $(\theta_1^{(t-1)},\cdots, \theta_N^{(t-1)})$ before step $t$ under $D$ and $D'$, respectively.

Conditional on $H$, the combined update has the form
\[
\theta^{(t)}-\theta^{(t-1)}
=
\sum_{i=1}^N
\lambda_i\eta_i^{(t)} \widetilde g_i^{(t)}(D;H)
+
\overline e_\lambda^{(t)},
\]
where $\widetilde g_i(D; H)$ is clipped gradient for model $i$ and
\[
\overline e_\lambda^{(t)}
=
\sum_{i=1}^N \lambda_i\eta_i^{(t)} e_i^{(t)}
\sim
\mathcal N(0,(s_\lambda^{(t)})^2 I),
\qquad
s_\lambda^2
=
\sum_{i=1}^N
(\lambda_i\eta_i^{(t)}\sigma_i^{(t)} C_i^{(t)})^2 .
\]
Thus, conditional on $H$, the distribution of $\theta^{(t)}-\theta^{(t-1)}$ under $D$ is Gaussian with covariance $(s_\lambda^{(t)})^2 I$. Similarly, conditional on $H'$, the distribution under $D'$ is Gaussian with the same covariance.
By clipping,
\[
\left\|
\sum_{i=1}^N
\lambda_i\eta_i^{(t)}
\left(
\widetilde g_i^{(t)}(D;H)-\widetilde g_i^{(t)}(D';H')
\right)
\right\|_2
\le
2\sum_{i=1}^N \lambda_i\eta_i^{(t)} C_i^{(t)}
=
\Delta_\lambda ^{(t)}.
\]

We first bound the hockey-stick divergence conditional on $H$ and $H'$. 
Conditioned on $H,H'$ and the neighboring datasets $D,D'$, the two conditional distributions differ only through their Gaussian means and share the same covariance $(s_\lambda^{(t)})^2 I$. 
Let
\[
v=v(H,H',D,D')
\]
denote the corresponding mean-difference vector. By the preceding sensitivity bound, 
\[
\|v\|_2 \le \Delta_\lambda^{(t)}.
\]
Let $F_{D,H}$ and $F_{D',H'}$ denote the conditional distributions of 
$\theta^{(t)}-\theta^{(t-1)}$ under $(D,H)$ and $(D',H')$, respectively. 
By translation invariance, the privacy loss random variable for the pair
$(F_{D,H},F_{D',H'})$ can be written as
\[
L_v^{(t)}
=
\log \frac{\phi(Y;v,(s_\lambda^{(t)})^2 I)}
{\phi(Y;0,(s_\lambda^{(t)})^2 I)}
=
\frac{\langle Y,v\rangle}{(s_\lambda^{(t)})^2}
-
\frac{\|v\|_2^2}{2(s_\lambda^{(t)})^2},
\qquad
Y\sim \mathcal N(0,(s_\lambda^{(t)})^2 I).
\]
Equivalently, letting $Z\sim\mathcal N(0,(s_\lambda^{(t)})^2)$ and
$r=\|v\|_2$, we have
\[
L_v^{(t)}
\overset{d}{=}
\frac{rZ}{(s_\lambda^{(t)})^2}
-
\frac{r^2}{2(s_\lambda^{(t)})^2}
=
\log
\frac{
\phi(Z;r,(s_\lambda^{(t)})^2)
}{
\phi(Z;0,(s_\lambda^{(t)})^2)
},
\]
where $\phi$ denotes the one-dimensional Gaussian density.

Therefore, by the definition of the hockey-stick divergence,
\[
\mathcal H_\varepsilon(F_{D,H},F_{D',H'})
=
\mathbb E\!\left[(e^{L_v^{(t)}}-e^\varepsilon)_+\right],
\]
and, for the reverse direction,
\[
\mathcal H_\varepsilon(F_{D',H'},F_{D,H})
=
\mathbb E\!\left[(1-e^{\varepsilon+L_v^{(t)}})_+\right].
\]

Let $s=s_\lambda^{(t)}$ and, for $r\ge 0$, define
\[
L_r
=
\frac{rZ}{s^2}-\frac{r^2}{2s^2},
\qquad
Z\sim\mathcal N(0,s^2).
\]
We first show that
$
\mathbb E\!\left[(e^{L_r}-e^\varepsilon)_+\right]
$
is nondecreasing in $r$. For $r>0$,
\[
e^{L_r}>e^\varepsilon
\iff
Z>a_r,
\qquad
a_r:=\frac{\varepsilon s^2}{r}+\frac r2 .
\]
Therefore,
\[
\mathbb E\!\left[(e^{L_r}-e^\varepsilon)_+\right]
=
\int_{a_r}^{\infty}
\left(
e^{\frac{rz}{s^2}-\frac{r^2}{2s^2}}
-
e^\varepsilon
\right)
\phi(z;0,s^2)\,dz .
\]
Using
\[
e^{\frac{rz}{s^2}-\frac{r^2}{2s^2}}
\phi(z;0,s^2)
=
\phi(z;r,s^2),
\]
we can rewrite this as
\[
\mathbb E\!\left[(e^{L_r}-e^\varepsilon)_+\right]
=
\int_{a_r}^{\infty}
\left(
\phi(z;r,s^2)
-
e^\varepsilon \phi(z;0,s^2)
\right) dz .
\]
Differentiating with respect to $r$, the boundary term vanishes because the integrand is zero at $z=a_r$. Thus,
\[
\frac{d}{dr}
\mathbb E\!\left[(e^{L_r}-e^\varepsilon)_+\right]
=
\int_{a_r}^{\infty}
\frac{\partial}{\partial r}
\phi(z;r,s^2)\,dz .
\]
Since
\[
\frac{\partial}{\partial r}\phi(z;r,s^2)
=
-\frac{\partial}{\partial z}\phi(z;r,s^2),
\]
we obtain
\[
\frac{d}{dr}
\mathbb E\!\left[(e^{L_r}-e^\varepsilon)_+\right]
=
\phi(a_r;r,s^2)\ge 0.
\]
Hence $\mathbb E[(e^{L_r}-e^\varepsilon)_+]$ is nondecreasing in $r$. 

Similarly, we show that
$
\mathbb E\!\left[(1-e^{\varepsilon+L_r})_+\right]
$
is nondecreasing in $r$. Since
\[
1-e^{\varepsilon+L_r}>0
\iff
L_r<-\varepsilon
\iff
\frac{rZ}{s^2}-\frac{r^2}{2s^2}<-\varepsilon,
\]
for $r>0$ this is equivalent to
\[
Z < b_r,
\qquad
b_r:=\frac{r}{2}-\frac{\varepsilon s^2}{r}.
\]
Therefore,
\[
\mathbb E\!\left[(1-e^{\varepsilon+L_r})_+\right]
=
\int_{-\infty}^{b_r}
\left(
1-e^\varepsilon e^{\frac{rz}{s^2}-\frac{r^2}{2s^2}}
\right)
\phi(z;0,s^2)\,dz .
\]
Using
\[
e^{\frac{rz}{s^2}-\frac{r^2}{2s^2}}\phi(z;0,s^2)
=
\phi(z;r,s^2),
\]
we can rewrite this as
\[
\mathbb E\!\left[(1-e^{\varepsilon+L_r})_+\right]
=
\int_{-\infty}^{b_r}
\left(
\phi(z;0,s^2)
-
e^\varepsilon \phi(z;r,s^2)
\right)dz .
\]
Differentiating with respect to $r$, the boundary term vanishes because the integrand is zero at $z=b_r$. Thus,
\[
\frac{d}{dr}
\mathbb E\!\left[(1-e^{\varepsilon+L_r})_+\right]
=
-e^\varepsilon
\int_{-\infty}^{b_r}
\frac{\partial}{\partial r}\phi(z;r,s^2)\,dz .
\]
Since
\[
\frac{\partial}{\partial r}\phi(z;r,s^2)
=
-\frac{\partial}{\partial z}\phi(z;r,s^2),
\]
we obtain
\[
\frac{d}{dr}
\mathbb E\!\left[(1-e^{\varepsilon+L_r})_+\right]
=
e^\varepsilon
\int_{-\infty}^{b_r}
\frac{\partial}{\partial z}\phi(z;r,s^2)\,dz
=
e^\varepsilon \phi(b_r;r,s^2)
\ge 0.
\]
Hence,
\[
\mathbb E\!\left[(1-e^{\varepsilon+L_r})_+\right]
\]
is also nondecreasing in $r$.

Recall $r\le \Delta_\lambda^{(t)}$, $\mathbb E\!\left[(1-e^{\varepsilon+L_r})_+\right],\mathbb E[(e^{L_r}-e^\varepsilon)_+]$ are upper bounded by the corresponding quantities with
$r=\Delta_\lambda^{(t)}$.

Since $L_r \overset{d}{=} L_v^{(t)}$, we know
$$\mathbb E\!\left[(1-e^{\varepsilon+L_r})_+\right]=\mathbb E\!\left[(1-e^{\varepsilon+L_v^{(t)}})_+\right], \qquad \mathbb E[(e^{L_r}-e^\varepsilon)_+]=\mathbb E[(e^{L_v^{(t)}}-e^\varepsilon)_+]$$.

Hence, defining
\[
\widetilde L^{(t)}
=
\log
\frac{
\phi(Z;\Delta_\lambda^{(t)},(s_\lambda^{(t)})^2)
}{
\phi(Z;0,(s_\lambda^{(t)})^2)
},
\qquad
Z\sim\mathcal N(0,(s_\lambda^{(t)})^2),
\]
we obtain, uniformly over all $H,H',D,D'$,
\[
\mathcal H_\varepsilon(F_{D,H},F_{D',H'})
\le
\mathbb E\!\left[(e^{\widetilde L^{(t)}}-e^\varepsilon)_+\right],
\]
and
\[
\mathcal H_\varepsilon(F_{D',H'},F_{D,H})
\le
\mathbb E\!\left[(1-e^{\varepsilon+\widetilde L^{(t)}})_+\right].
\]

Now condition on the  $\mathcal F_{t-1}$. 
The conditional distribution under $D$ can be written as a mixture over $H$ and $H'$:
\[
P^{(t)}
=
\int F_{D,H}\,d\kappa(H\mid D,\mathcal F_{t-1}),
\]
and similarly
\[
Q^{(t)}
=
\int F_{D',H'}\,d\kappa(H'\mid D',\mathcal F_{t-1}),
\]
where $\kappa(\cdot\mid D,\mathcal F_{t-1})$ denotes the regular conditional distribution of the hidden states given the previously accounted transcript under dataset $D$.

By joint convexity of the hockey-stick divergence,
\[
\mathcal H_\varepsilon(P^{(t)},Q^{(t)})
\le
\iint
\mathcal H_\varepsilon
\left(
F_{D,H},F_{D',H'}
\right)
\,d\kappa(H\mid D,\mathcal F_{t-1})
\,d\kappa(H'\mid D',\mathcal F_{t-1}).
\]
Therefore,
\[
\mathcal H_\varepsilon(P^{(t)},Q^{(t)})
\le
\mathbb E[(e^{\widetilde L^{(t)}}-e^\varepsilon)_+].
\]
The same argument with the two distributions swapped gives
\[
\mathcal H_\varepsilon(Q^{(t)},P^{(t)})
\le
\mathbb E[(1-e^{\varepsilon+\widetilde L^{(t)}})_+].
\]
This proves the theorem.
\end{proof}

\begin{proof}[Proof of Corollary \ref{cor:pld-linmix-compose}]
We consider the following lemma.
\begin{lemma}\label{lem:pld-linmix-compose}
    Let $L^{(t)}$ be privacy loss random variables step t ($t\in [1\dots T]$), conditioned on $\mathcal F_{t-1}$, and $\tilde L^{(t)}$ is surrogate. Define $$
\delta_{L^{(t)}}(\varepsilon):=\mathbb E[(1-e^{\varepsilon-L^{(t)}})_+],
\qquad
\delta_{\tilde L^{(t)}}(\varepsilon)
:=
\mathbb E\bigl[(1-e^{\varepsilon-\tilde L^{(t)}})_+\bigr].
$$

Suppose that $\forall t$
$$
\delta_{\tilde L^{(t)}}(\varepsilon)\ge \delta_{L^{(t)}}(\varepsilon),
\qquad \forall \varepsilon \in \mathbb{R}.
$$
Then
$$
\delta_{\tilde L^{(1)}+\cdots+\tilde L^{(T)}}(\varepsilon)
\ge
\delta_{L^{(1)}+\cdots+L^{(T)}}(\varepsilon),
\qquad \forall \varepsilon\in\mathbb{R}.
$$
\end{lemma}

\begin{proof}
    For $T=2$, note
    $$\delta_{L^{(1)}+L^{(2)}}(\varepsilon)=\mathbb E[e^{L^{(1)}}\mathbb E[(e^{L^{(2)}}-e^{\varepsilon-L^{(1)}})^+]\mid F_1]=\mathbb E[e^{L^{(1)}}\delta_{L^{(2)}}(\varepsilon -L^{(1)})]\le\mathbb E[e^{L^{(1)}}\delta_{\tilde L^{(2)}}(\varepsilon -L^{(1)})]=\delta_{L^{(1)}+\tilde L^{(2)}}(\varepsilon).$$
By the construction, $\tilde L^{(2)}$ only depend on randomness in step 2 and thus is independent to $L^{(1)}$. So similarly, $\delta_{L^{(1)}+\tilde L^{(2)}}(\varepsilon)\le\delta_{\tilde L^{(1)}+\tilde L^{(2)}}(\varepsilon)$.
Thus
$$\delta_{L^{(1)}+L^{(2)}}(\varepsilon)\le\delta_{\tilde L^{(1)}+\tilde L^{(2)}}(\varepsilon).$$
This proves the result for $T=2$. 

Assume now that the statement holds for $T=k-1$. For $T=k$, let
$$
O_1=L^{(1)}+\cdots+L^{(k-1)},\qquad O_2=L^{(k)},
$$
and define $O_1',O_2'$ analogously. By the induction hypothesis, $\delta_{O_1'}(\varepsilon)\ge \delta_{O_1}(\varepsilon)$ for all $\varepsilon$. Applying the two-term result to $(O_1,O_2)$ and $(O_1',O_2')$ yields the claim for $T=k$. By induction, we finish the proof.
\end{proof}

Corollary \ref{cor:pld-linmix-compose} is a direct application of Lemma \ref{lem:pld-linmix-compose}. We finish the proof.
\end{proof}

\subsection{Proof of Proposition \ref{prop:rdp-vs-drv10-gaussian}} \label{app:proof:tighter}
\begin{proof}
WLOG, we assume $\frac{\Delta}{\sigma}=1$

\textbf{Step 1 Per-release RDP parameter.}
For the Gaussian mechanism (mean estimation with $\ell_2$-sensitivity $\Delta$ and noise
$Z\sim\mathcal{N}(0,\sigma^2)$), the per-release Rényi DP parameter at order $\alpha>1$ is
\[
\varepsilon_{\alpha} \;=\; \frac{\alpha\,\Delta^2}{2\sigma^2}.
\]

\textbf{Step 2 Per-release $(\varepsilon,\delta)$ from RDP.}
Since
\[
\varepsilon \;=\; \inf_{\alpha>1}\left\{\varepsilon_{\alpha}+\frac{\log(1/\delta)}{\alpha-1}\right\},
\]
We know
$$
\varepsilon =\frac{\Delta^2}{2\sigma^2}+\frac{\,\Delta}{\sigma}\sqrt{2\log(1/\delta)}
$$

\textbf{Step 3 composition in $(\varepsilon,\delta)$.}
Applying \citep[\textnormal{Theorem~III.3}]{dwork2010boosting} to $N$ releases, each satisfying
$(\varepsilon,\delta)$-DP, yields a composed guarantee $(\varepsilon_{\mathrm{com}},\delta')$ with
\[
\delta' \;=\; N\delta+\delta_0,
\]
and
$$\varepsilon_{\mathrm{com}}=\varepsilon\sqrt{2N\log(1/\delta_0)}+N\varepsilon (e^{\varepsilon}-1)$$.

\textbf{Step 4 Joint RDP accounting.}
By additivity of RDP under composition, we know
\[
\varepsilon_{\mathrm{RDP}}
\;=\;
\inf_{\alpha>1}\left\{N\varepsilon_{\alpha}+\frac{\log(1/\delta')}{\alpha-1}\right\}=\inf_{\alpha>1}\left\{\frac{N\alpha\,\Delta^2}{2\sigma^2}+\frac{\log(1/\delta')}{\alpha-1}\right\}
\]
Thus
$$
\varepsilon_{\mathrm{RDP}}=\frac{N\Delta^2}{2\sigma^2}+\frac{\Delta}{\sigma}\sqrt{2N\log(1/\delta')}.
$$
We only need to prove 
$$
 \frac{N\Delta^2}{2\sigma^2}+\frac{\Delta}{\sigma}\sqrt{2N\log(1/\delta')} \leq \varepsilon\sqrt{2N\log(1/\delta_0)}+N\varepsilon (e^{\varepsilon}-1).
$$
Let $t=\frac{\Delta}{\sigma}, s=\sqrt{2\log(1/\delta)}$.
Bringing
$$
\varepsilon =\frac{1}{2}t^2+st
$$
into the inequality, and use $e^\varepsilon>1+\varepsilon$, since $e^\varepsilon>1+\varepsilon$, we only need to prove
$$
\frac{N}{2}t^2+t\sqrt{2N\log(1/\delta')} \leq (\frac{1}{2}t^2+st)\sqrt{2N\log(1/\delta_0)}+N(\frac{1}{2}t^2+st)^2
$$

Let $A\triangleq \sqrt{2N\log(1/\delta_0)}$ and $C\triangleq \sqrt{2N\log(1/\delta')}$.
Since $\delta' = N\delta+\delta_0\ge \delta_0$, we have $C\le A$. Therefore the left-hand side is at most $\frac{N}{2}t^2+tA$, and it is enough to show
\begin{equation}\label{eq:stronger}
\frac{N}{2}t^2+tA
\;\le\;
\Big(\frac12 t^2+st\Big)A
\;+\;
N\Big(\frac12 t^2+st\Big)^2 .
\end{equation}

Rearranging \eqref{eq:stronger}, the gap equals
\begin{align*}
&\Big(\frac12 t^2+st-t\Big)A
\;+\;
N\Big[\Big(\frac12 t^2+st\Big)^2-\frac12 t^2\Big] \\
&=
\Big(\frac12 t^2+(s-1)t\Big)A
\;+\;
N\Big(\frac14 t^4+s t^3+(s^2-\frac12)t^2\Big).
\end{align*}
Since $\delta\le 1/2$, we have $s=\sqrt{2\log(1/\delta)}\ge \sqrt{2\log 2}>1$,
hence both $s-1\ge 0$ and $s^2-\frac12\ge 0$. Since $t\ge 0$ and $A\ge 0$, every term on the
right-hand side is nonnegative, so the gap is $\ge 0$. We finish the proof.
\end{proof}

\section{Compare to Other Model Merging Methods}\label{app:model-merging}
Model merging has emerged as a widely used technique for combining multiple trained models, but focuses on non-private settings. In our paper, model merging refers to merging multiple trained models in parameter space.
Existing model merging methods can be broadly categorized into: (i) simple averaging methods such as model soups \citep{wortsman2022model}, which directly average model parameters with data-independent weights; our LC method is closest to this family, and privacy analysis is relatively straightforward in this case; (ii) data-dependent weighted averaging methods such as Fisher-weighted averaging \citep{matena2022merging}, which are similar to LC but typically require access to data to determine the weights, making privacy accounting more complex; (iii) parameter-dependent merging methods such as TIES ~\citep{yadav2023ties}, which determine the merge at the parameter level based on properties of the parameters themselves (e.g., sparsity or sign agreement); these methods often involve nonlinear operations, which make the merged model harder to analyze from a privacy perspective; and (iv) checkpoint merging methods such as SWA \citep{izmailov2018averaging}. 
Intuitively, in the DP-SGD setting, checkpoint merging usually does not provide a privacy guarantee better than the least private checkpoint, since later checkpoints already accumulate all privacy loss from earlier training steps, and merging checkpoints mainly changes how these updates are combined rather than reducing the underlying privacy cost, as discussed in Appendix~\ref{subsec:rs-same-run}. As a result, methods such as SWA can be useful for improving model utility, but are less suitable for our setting, where the goal is to achieve a better privacy-utility tradeoff. 

\section{Experiment Details}\label{app:experiment details}
We use 42 for random seed in all experiments. The main experiments were run on a single NVIDIA RTX 4090 GPU, and the total runtime varied from a few hours to one day across different datasets and model architectures.
\subsection{MNIST}
\paragraph{Dataset and Model.}
We use MNIST with the standard train/test split ($60{,}000/10{,}000$). Images are normalized. We train a logistic regression classifier:
\[
\texttt{Flatten} \rightarrow \texttt{Linear}(28\!\times\!28  \rightarrow 10),
\]
optimized with cross-entropy loss.
We report test-set classification accuracy (and loss) after each epoch and use the final test accuracy
for Pareto-frontier plots.
\paragraph{DP-SGD Implementation.}
We use the Opacus package \citep{yousefpour2021opacus} to train with DP-SGD.
The optimizer is SGD. In our experiments, we train three base models with the following hyperparameters:
\[
\text{max learning rate } 4,\quad C=2,\quad \sigma=32,\quad T=20\ \text{epochs};
\]
\[
\text{max learning rate } 4,\quad C=4,\quad \sigma=32,\quad T=20\ \text{epochs};
\]
\[
\text{max learning rate } 4,\quad C=2,\quad \sigma=64,\quad T=20\ \text{epochs}.
\]
We use $N=60{,}000$ training samples and full batch.
We use a linear learning-rate schedule with a $10\%$-epoch warmup followed by linear decay; see the code for implementation details.

\subsection{CIFAR-10}
\paragraph{Dataset and Model.}
We use CIFAR-10 with the standard split ($50{,}000/10{,}000$). We apply standard data preprocessing.
We use ResNet18 with CIFAR-10 adaptations:
the first convolution is replaced by a $3\times 3$ stride-$1$ convolution and the max-pooling layer is
removed. The final fully-connected layer is set to output $10$ classes.
We report test-set accuracy (and loss) after each epoch; the final test accuracy is used for the
accuracy--privacy Pareto-frontier plots.


\paragraph{DP-SGD from a Pretrained Model Initialization.}
In addition to training from scratch, we also consider a pretraining-based setup on CIFAR-10 to better reflect practical workflows. We randomly split the original training set into two disjoint subsets of equal size, denoted by $D_{\mathrm{pre}}$ and $D_{\mathrm{priv}}$. We first train a non-private ResNet18 model on $D_{\mathrm{pre}}$ using standard SGD, and then use the resulting checkpoint as initialization for DP-SGD on $D_{\mathrm{priv}}$. All merged candidate models are produced from this common pretrained initialization. The privacy guarantee in this setting is with respect to $D_{\mathrm{priv}}$ only, since $D_{\mathrm{pre}}$ is disjoint from the private training subset and is not included in the DP accounting. We fix the randomness
of data split using a generator seed of $42$. We train the pretrained model using standard SGD with batch size $B=128$ and learning rate $\eta=0.1$, and train for $T=100$ epochs.

Starting from the pretrained checkpoint, we fine-tune the model on $D_{\mathrm{priv}}$ using DP-SGD in Opacus with the same implementation adjustments as above, including \texttt{ModuleValidator.fix(model)} and setting all ReLU layers to \texttt{inplace=False}. We use the SGD optimizer with per-sample gradient clipping norm $C$, noise multiplier $\sigma$, full batch, learning rate $0.1$, and train for $T=20$ epochs. As in the from-scratch experiments, we fix the randomness of data shuffling using a DataLoader generator seed of $42$. We generate candidate models by sweeping the learning rate $C \in \{2,4\}$, and $\sigma \in \{16,32\}$. 

\subsection{RS for Merging Checkpoints from the Same Run}
\label{subsec:rs-same-run}

We note that the limitation of applying LC to recycled checkpoints is that under DP-SGD and standard privacy
accounting, it cannot yield a privacy guarantee better than the least private checkpoint (typically the checkpoint
taken at the largest training step). To see why, consider two checkpoints $\theta_1$ and $\theta_2$ taken from the
same DP-SGD run at steps $t_1<t_2$. A linear merge can be written as
\[
\theta \;=\; \lambda\theta_1+(1-\lambda)\theta_2
\;=\;
\theta_1+(1-\lambda)(\theta_2-\theta_1).
\]
Since $\theta_2-\theta_1$ is exactly the cumulative update from step $t_1$ to $t_2$, this merge is essentially
equivalent to re-scaling the updates on the segment $t\in\{t_1+1,\ldots,t_2\}$ by a factor $(1-\lambda)$. However, standard DP accounting for DP-SGD is driven by the sampling
rate, clipping bound, noise multiplier, and number of steps, and is typically independent of the learning rate.
Therefore, such a post-processing does not reduce the privacy cost attributed to the steps up to
$t_2$, and the resulting guarantee cannot be tighter than that of the checkpoint $\theta_2$ itself.
In other words, while checkpoint recycling can improve utilities, linear combination over correlated checkpoints is not able to achieve tighter privacy  relative to the least private model, without additional assumptions.

We also consider the checkpoint-merging setting studied in \citet{indri2023can,shejwalkar2022recycling}, where the
candidate models are checkpoints from a single DP-SGD training run and therefore are generally not
independent. This distinction matters: our LC accounting relies on independence noises across the input mechanisms, whereas RS does not. Since RS only requires access to the individual privacy profiles, it remains applicable even when the inputs are correlated checkpoints.

We run experiments under DP-SGD on both an MNIST and CIFAR. 
We construct three candidate models by taking checkpoints at training steps $T\in\{1,2,3\}$ for MNIST and at training steps $T\in\{12,16,20\}$ for CIFAR-10 with batch size $256$, learning rate $0.1$ for MNIST and $0.01$ for CIFAR, $\sigma=0.5$ for MNIST and $\sigma=0.6$ for CIFAR. The results are shown in
Figure~\ref{fig:pare-same-run}. We observe that RS still achieves a favorable privacy/utility tradeoff in this
correlated-checkpoint regime.

\begin{figure}[h]
    \centering
    \begin{subfigure}[b]{0.4\textwidth}
        \centering
        \includegraphics[width=0.95\linewidth]{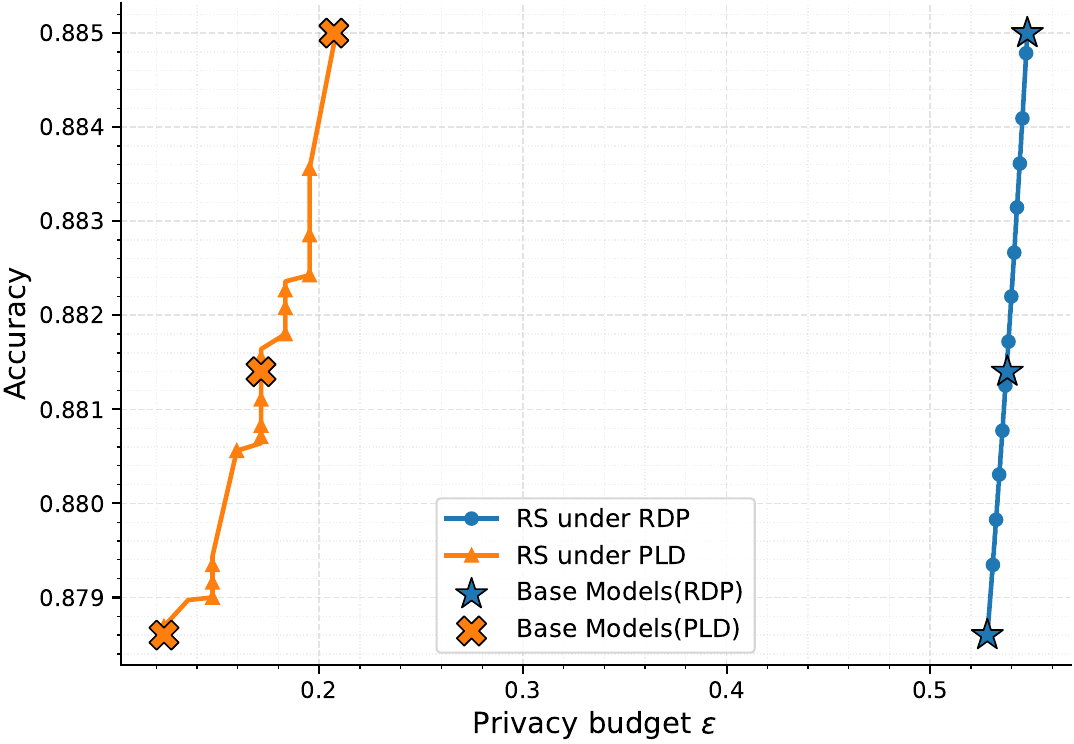}
    \caption{Privacy/utility tradeoffs on MNIST within same run ($\delta=10^{-5}$)}\label{fig:pare-MNIST-same-run}

    \end{subfigure}
    \hfill
    \begin{subfigure}[b]{0.4\textwidth}
        \centering
        \includegraphics[width=0.95\linewidth]{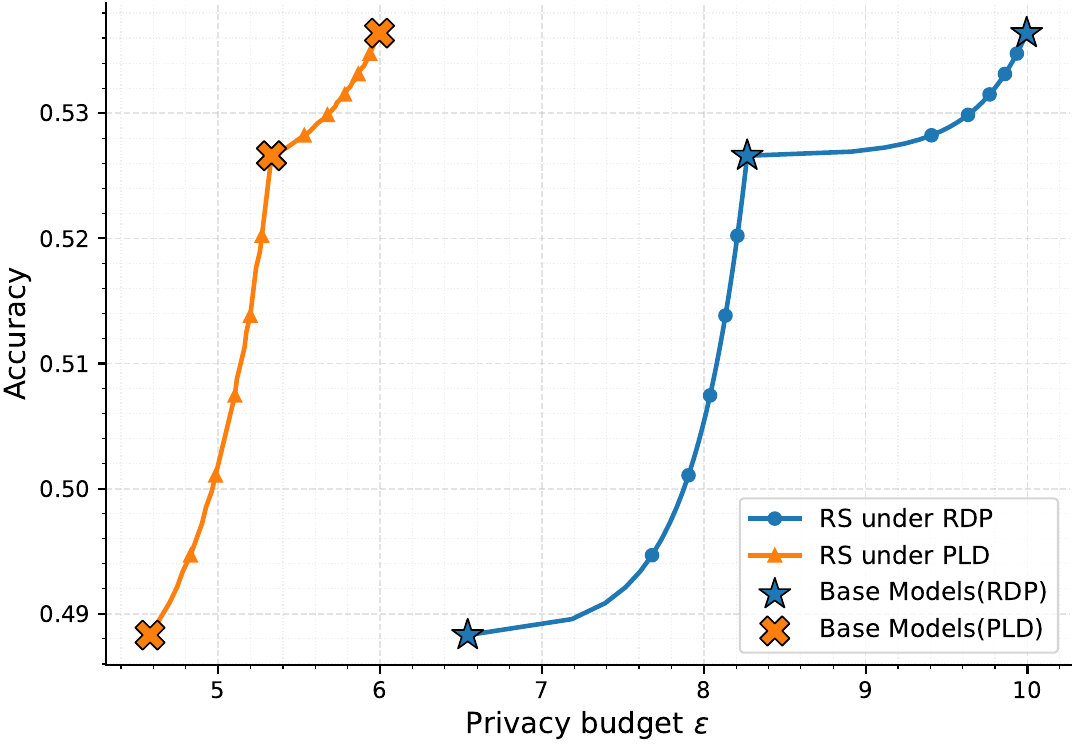}
    \caption{Privacy/utility tradeoffs on CIFAR-10 within same run ($\delta=10^{-5}$)}\label{fig:pare-CIFAR-same-run}
    \end{subfigure} 
    \caption{Privacy/utility tradeoffs of merging checkpoints from the same run.}\label{fig:pare-same-run}
\end{figure}

\subsection{More Experiment Results}
\begin{figure}[h]
    \centering
    \begin{subfigure}[b]{0.4\textwidth}
        \centering
        \includegraphics[width=\linewidth]{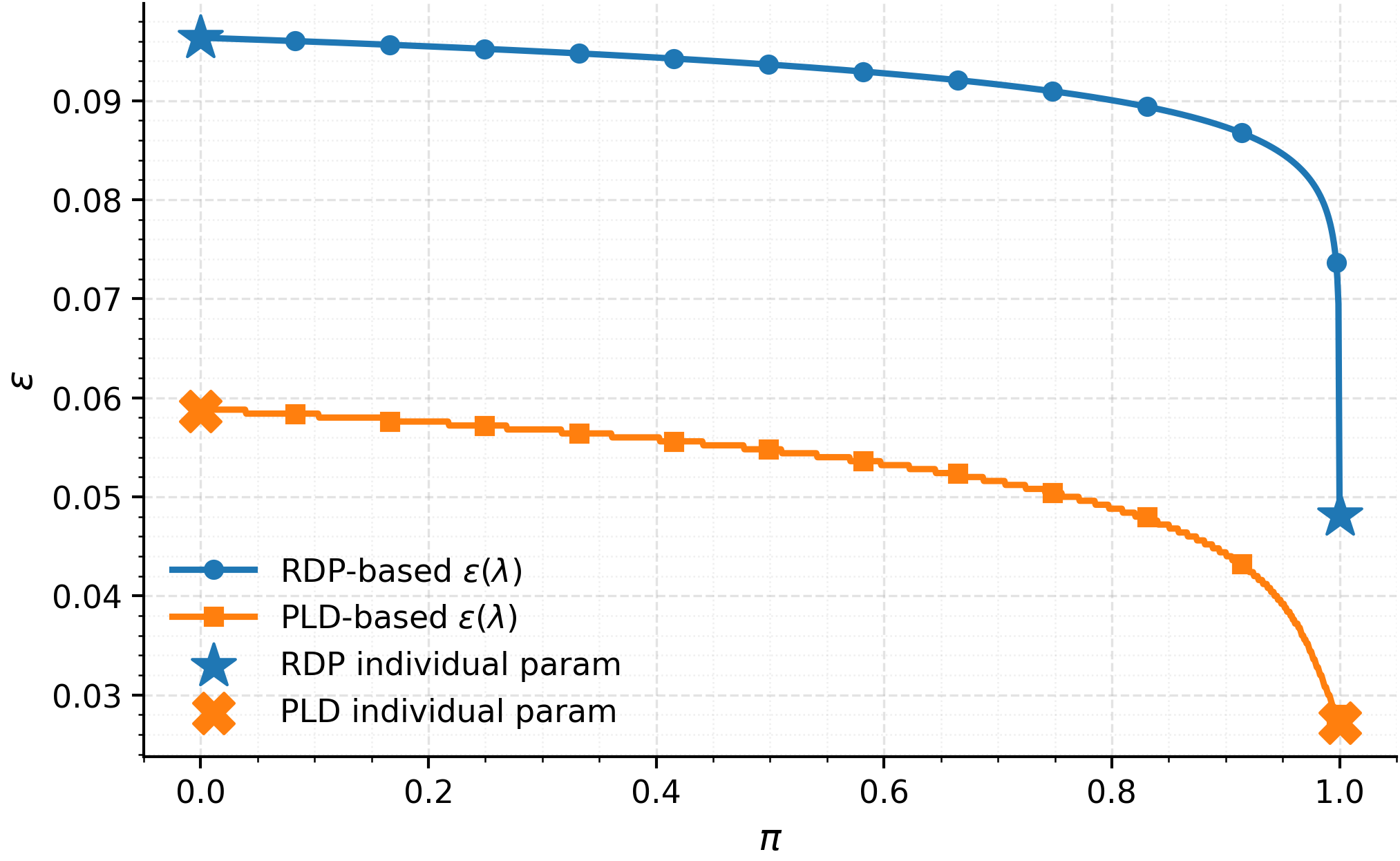}
        \caption{DP parameter as $\pi$ changes}

    \end{subfigure}
    \hfill
    \begin{subfigure}[b]{0.4\textwidth}
        \centering
        \includegraphics[width=\linewidth]{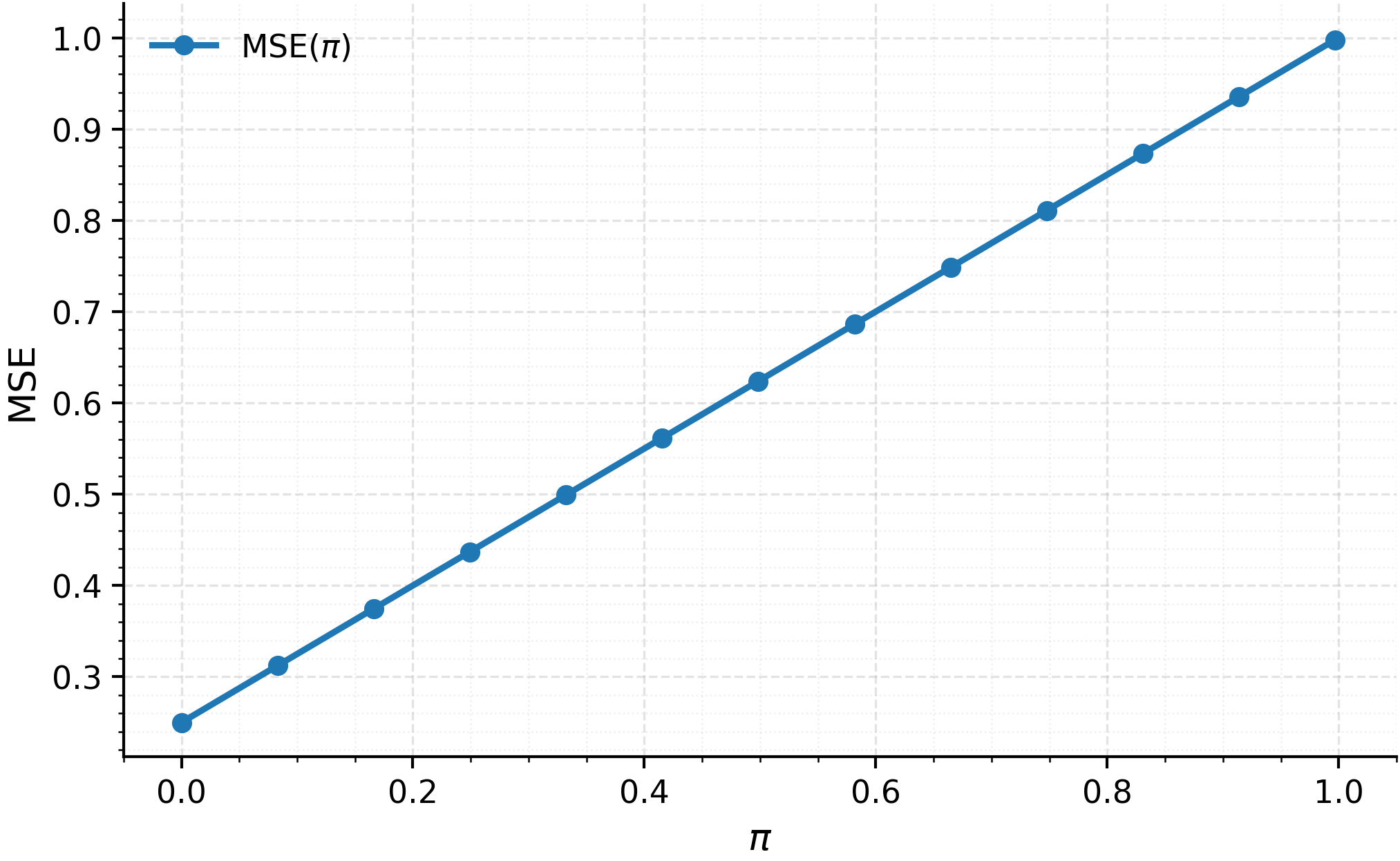}
        \caption{MSE as $\pi$ changes}
    \end{subfigure} 
    \caption{Random selection results for mean estimation with $\delta=10^{-5}$.} 
\end{figure}

\begin{figure}[h]
    \centering
    \begin{subfigure}[b]{0.4\textwidth}
        \centering
        \includegraphics[width=\linewidth]{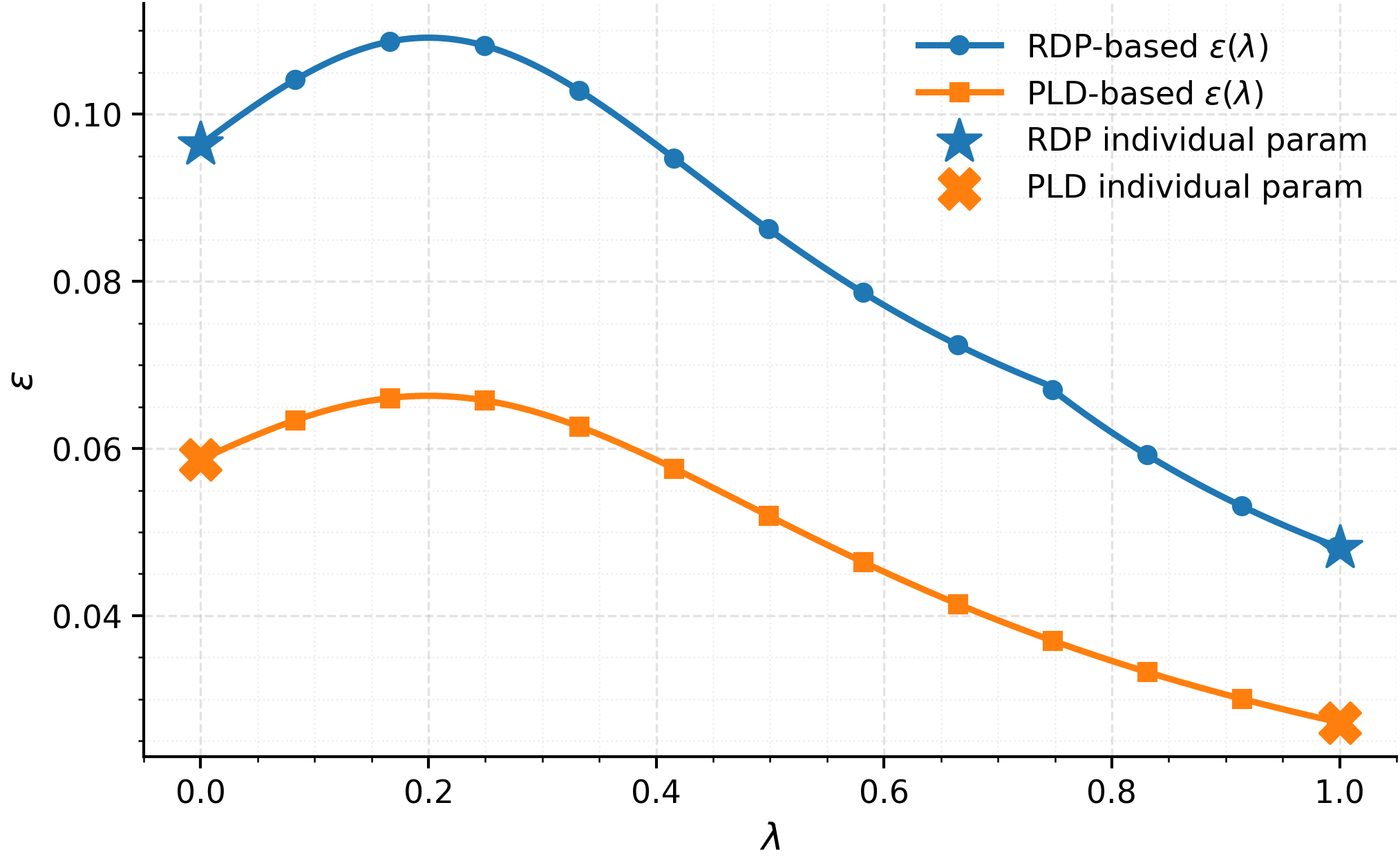}
        \caption{DP parameter as $\lambda$ changes}

    \end{subfigure}
    \hfill
    \begin{subfigure}[b]{0.4\textwidth}
        \centering
        \includegraphics[width=\linewidth]{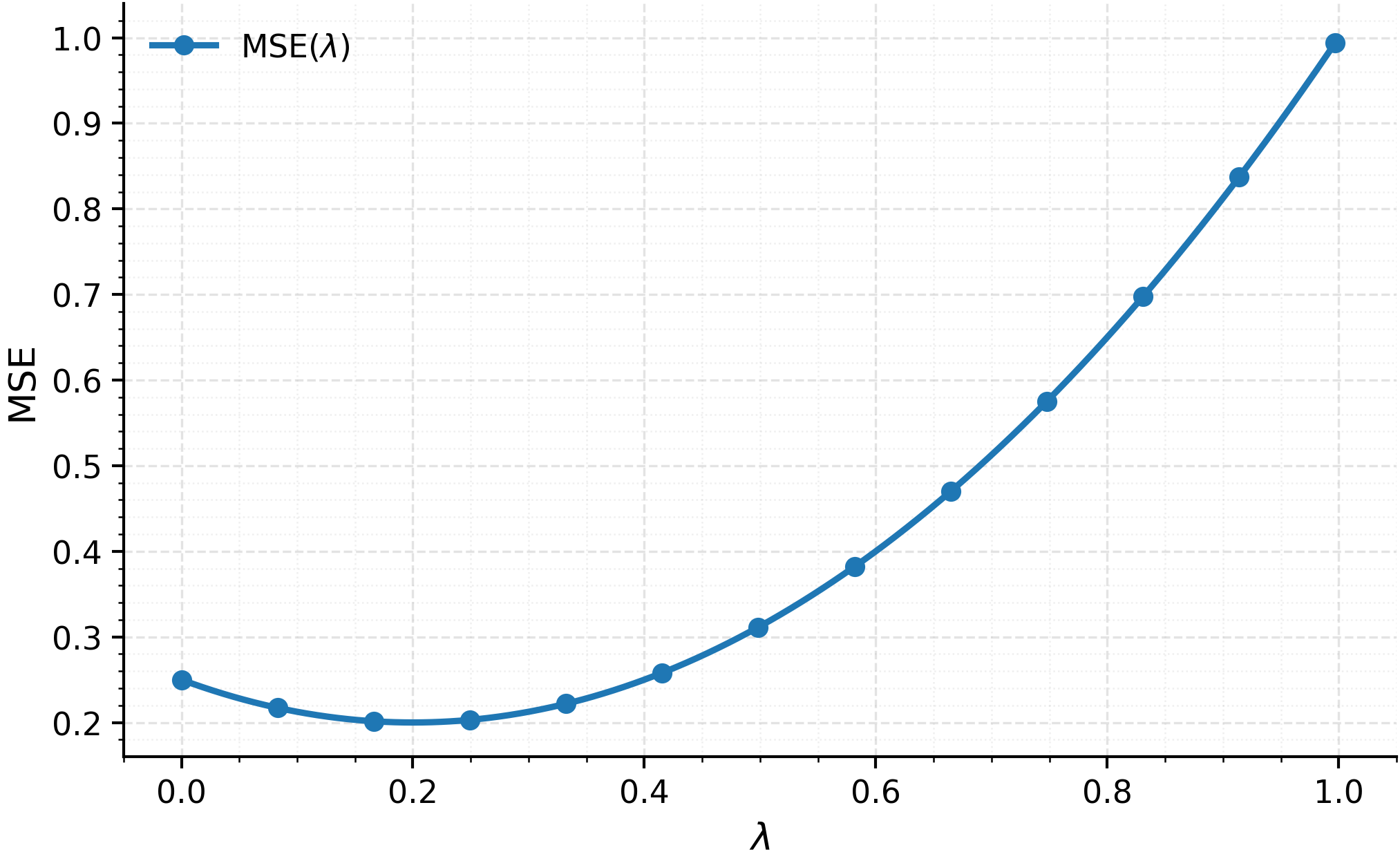}
        \caption{MSE as $\lambda$ changes}
    \end{subfigure} 
    \caption{Linear combination results for mean estimation with $\delta=10^{-5}$.}
\end{figure}


\end{document}